\documentclass{article}

    \PassOptionsToPackage{numbers, compress}{natbib}

    \usepackage[final]{neurips_2021}

\usepackage[utf8]{inputenc} 
\usepackage[T1]{fontenc}    
\usepackage{hyperref}       
\usepackage{url}            
\usepackage{booktabs}       
\usepackage{amsfonts}       
\usepackage{nicefrac}       
\usepackage{microtype}      
\usepackage{xcolor}         

\usepackage{appendix}
\usepackage{multirow}
\usepackage{graphicx}
\graphicspath{ {./images/} }
\usepackage{enumitem}
\setlist{topsep=0pt, leftmargin=*}

\usepackage{bm}
\usepackage{bbm}
\usepackage{amsmath}
\usepackage{amssymb}
\usepackage{amsthm}
\usepackage{mathtools}
\newtheorem{theorem}{Theorem}[section]
\newtheorem{corollary}{Corollary}[theorem]
\newtheorem{lemma}[theorem]{Lemma}
\newtheorem{proposition}{Proposition}[section]
\theoremstyle{definition}
\newtheorem{definition}{Definition}[section]
\theoremstyle{remark}
\newtheorem{remark}{Remark}[section]

\DeclarePairedDelimiter\abs{\lvert}{\rvert}
\DeclarePairedDelimiter{\norm}{\lVert}{\rVert}
\DeclarePairedDelimiter{\inner}{\langle}{\rangle}

\newcommand{\expect}[1]{\mathop{\mathbb{E}}#1}
\newcommand{\bracket}[1]{\left[#1\right]}
\newcommand{\parentheses}[1]{\left(#1\right)}

\DeclareMathOperator*{\argmin}{arg\,min}

\title{On the Equivalence between Neural Network and Support Vector Machine}

\author{%
  Yilan Chen \\
  Computer Science and Engineering\\
  University of California San Diego\\
  La Jolla, CA \\
  \texttt{yilan@ucsd.edu} \\
   \And
   Wei Huang \\
   Engineering and Information Technology\\
   University of Technology Sydney \\
   Ultimo, Australia \\
   \texttt{weihuang.uts@gmail.com} \\
   \AND
   Lam M. Nguyen \\
   IBM Research \\
   Thomas J. Watson Research Center \\
   Yorktown Heights, NY \\
   \texttt{LamNguyen.MLTD@ibm.com} \\
   \And
   Tsui-Wei Weng \\
  Halıcıoğlu Data Science Institute\\
  University of California San Diego\\
  La Jolla, CA \\
  \texttt{lweng@ucsd.edu} \\
}

\begin{document}

\maketitle

\begin{abstract}
    Recent research shows that the dynamics of an infinitely wide neural network (NN) trained by gradient descent can be characterized by Neural Tangent Kernel (NTK) \citep{jacot2018neural}. Under the squared loss, the infinite-width NN trained by gradient descent with an infinitely small learning rate is equivalent to kernel regression with NTK \citep{arora2019exact}. However, the equivalence is only known for ridge regression currently \citep{arora2019harnessing}, while the equivalence between NN and other kernel machines (KMs), e.g. support vector machine (SVM), remains unknown. Therefore, in this work, we propose to establish the equivalence between NN and SVM, and specifically, the infinitely wide NN trained by soft margin loss and the standard soft margin SVM with NTK trained by subgradient descent. Our main theoretical results include establishing the equivalences between NNs and a broad family of $\ell_2$ regularized KMs with finite-width bounds, which cannot be handled by prior work, and showing that every finite-width NN trained by such regularized loss functions is approximately a KM. Furthermore, we demonstrate our theory can enable three practical applications, including (i) \textit{non-vacuous} generalization bound of NN via the corresponding KM; (ii) \textit{non-trivial} robustness certificate for the infinite-width NN (while existing robustness verification methods would provide vacuous bounds); (iii) intrinsically more robust infinite-width NNs than those from previous kernel regression. Our code for the experiments is available at \url{https://github.com/leslie-CH/equiv-nn-svm}.

\end{abstract}

\section{Introduction}
Recent research has made some progress towards deep learning theory from the perspective of infinite-width neural networks (NNs). For a fully-trained infinite-width NN, it follows kernel gradient descent in the function space with respect to Neural Tangent Kernel (NTK)~\citep{jacot2018neural}. Under this linear regime and squared loss, it is proved that the fully-trained NN is equivalent to kernel regression with NTK \citep{jacot2018neural, arora2019exact}, which gives the generalization ability of such a model \citep{arora2019fine}. NTK helps us understand the optimization \citep{jacot2018neural,du2019gradient} and generalization \citep{arora2019fine,cao2019generalization} of NNs through the perspective of kernels. However, existing theories about NTK \citep{jacot2018neural, lee2019wide, arora2019exact, chizat2018lazy} usually assume the loss is a function of the model output, which does not include the case of regularization. Besides, they usually consider the squared loss that corresponds to a kernel regression, which may have limited insights to understand classification problems since squared loss and kernel regression are usually used for regression problems. 








On the other hand, another popular machine learning paradigm with solid theoretical foundation before the prevalence of deep neural networks is the support vector machine (SVM) \citep{boser1992training, cortes1995support}, which allows learning linear classifiers in high dimensional feature spaces. SVM tackles the sample complexity challenge by searching for large margin separators and tackles the computational complexity challenge using the idea of kernels \citep{shalev2014understanding}. To learn an SVM model, it usually involves solving a dual problem which is cast as a convex quadratic programming problem. Recently, there are some algorithms using subgradient descent \citep{shalev2011pegasos} and coordinate descent \citep{hsieh2008dual} to further scale the SVM models to large datasets and high dimensional feature spaces. 




We noticed that existing theoretical analysis mostly focused on connecting NN with kernel regression \citep{jacot2018neural, arora2019exact, lee2019wide} but the connections between NN and SVM have not yet been explored. In this work, we establish the equivalence between NN and SVM for the first time to our best knowledge. More broadly, we show that our analysis can connect NNs with a family of $\ell_2$ regularized KMs, including kernel ridge regression (KRR), support vector regression (SVR) and $\ell_2$ regularized logistic regression, where previous results \citep{jacot2018neural, arora2019exact, lee2019wide} cannot handle. These are the equivalences beyond ridge regression for the first time. Importantly, the equivalences between infinite-width NNs and these $\ell_2$ regularized KMs may shed light on the understanding of NNs from these new equivalent KMs~\citep{cristianini2000introduction, shawe2004kernel, scholkopf2002learning, steinwart2008support}, especially towards understanding the training, generalization, and robustness of NNs for classification problems. Besides, regularization plays an important role in machine learning to restrict the complexity of models. These equivalences may shed light on the understanding of the regularization for NNs. We highlight our contributions as follows:

\begin{itemize}
    \item We derive the continuous (gradient flow) and discrete dynamics of SVM trained by subgradient descent and the dynamics of NN trained by soft margin loss. We show the dynamics of SVM with NTK and NN are exactly the same in the infinite width limit because of the constancy of the tangent kernel and thus establish the equivalence. We show same linear convergence rate of SVM and NN under reasonable assumption. We verify the equivalence by experiments of subgradient descent and stochastic subgradient descent on MNIST dataset \citep{lecun1998gradient}.
    
    
    \item We generalize our theory to general loss functions with $\ell_2$ regularization and establish the equivalences between NNs and a family of $\ell_2$ regularized KMs as summarized in Table~\ref{table:loss_kernel}. We prove the difference between the outputs of NN and KM sacles as $O(\ln{m}/\lambda \sqrt{m})$, where $\lambda$ is the coefficient of the regularization and $m$ is the width of the NN. Additionally, we show every finite-width neural network trained by a $\ell_2$ regularized loss function is approximately a KM.
    
    
    \item We show that our theory offers three practical benefits: (i) computing \textit{non-vacuous} generalization bound of NN via the corresponding KM; (ii) we can deliver \textit{nontrivial} robustness certificate for the over-parameterized NN (with width $m \rightarrow \infty$) while existing robustness verification methods would give trivial robustness certificate due to bound propagation~\citep{gowal2018effectiveness, weng2018towards, zhang2018efficient}. In particular, the certificate decreases at a rate of $O(1 /\sqrt{m})$ as the width of NN increases; (iii) we show that the equivalent infinite-width NNs trained from our $\ell_2$ regularized KMs are more robust than the equivalent NN trained from previous kernel regression \citep{jacot2018neural, arora2019exact} (see Table~\ref{table:Robustness_kernel_machines}), which is perhaps not too surprising as the regularization has a strong connection to robust machine learning.
    
\end{itemize}

\section{Related Works and Background}
\subsection{Related Works}
\textbf{Neural Tangent Kernel and dynamics of neural networks}. NTK was first introduced in \citep{jacot2018neural} and extended to Convolutional NTK \citep{arora2019exact} and Graph NTK \citep{du2019graph}. \cite{huang2020neural} studied the NTK of orthogonal initialization. \cite{arora2019harnessing} reported strong performance of NTK on small-data tasks both for kernel regression and kernel SVM. However, the equivalence is only known for ridge regression currently, but not for SVM and other KMs. A line of recent work \citep{du2018gradient, allen2019convergence} proved the convergence of (convolutional) neural networks with large but finite width in a non-asymptotic way by showing the weights do not move far away from initialization in the optimization dynamics (trajectory). \cite{lee2019wide} showed the dynamics of wide neural networks are governed by a linear model of first-order Taylor expansion around its initial parameters. However, existing theories about NTK \citep{jacot2018neural, lee2019wide, arora2019exact} usually assume the loss is a function of the model output, which does not include the case of regularization. Besides, they usually consider the squared loss that corresponds to a kernel regression, which may have limited insights to understand classification problems since squared loss and kernel regression are usually used for regression problems. In this paper, we study the regularized loss functions and establish the equivalences with KMs beyond kernel regression and regression problems. 

Besides, we studied the robustness of NTK models. \cite{hu2021regularization} studied the label noise (the labels are generated by a ground truth function plus a Gaussian noise) while we consider the robustness of input perturbation. They study the convergence rate of NN trained by $\ell_2$ regularized squared loss to an underlying true function, while we give explicit robustness certificates for NNs. Our robustness certificate enables us to compare different models and show the equivalent infinite-width NNs trained from our $\ell_2$ regularized KMs are more robust than the equivalent NN trained from previous kernel regression.



\textbf{Neural network and support vector machine}.
Prior works \citep{tang2013deep, sun2016depth, liu2016large, sokolic2017robust, liang2017soft} have explored the benefits of encouraging large margin in the context of deep networks. \cite{cho2012kernel} introduced a new family of positive-definite kernel functions that mimic the computation in multi-layer NNs and applied the kernels into SVM. \cite{domingos2020every} showed that NNs trained by gradient flow are approximately "kernel machines" with the weights and bias as functions of input data, which however can be much more complex than a typical kernel machine. \cite{shalev2011pegasos} proposed a subgradient algorithm to solve the primal problem of SVM, which can obtain a solution of accuracy $\epsilon$ in $\tilde{O}(1/\epsilon)$ iterations, where $\tilde{O}$ omits the logarithmic factors. In this paper, we also consider the SVM trained by subgradient descent and connect it with NN trained by subgradient descent. \cite{smola1998connection, andras2002equivalence} studied the connection between SVM and regularization neural network \citep{poggio1990networks}, one-hidden layer NN that has very similar structures with that of KMs and is not widely used in practice. NNs used in practice now (e.g. fully connected ReLU NN, CNN, ResNet) do not have such structures. \cite{pellegrini2020analytic} analyzed two-layer NN trained by hinge loss without regularization on linearly separable dataset. Note for SVM, it must have a regularization term such that it can achieve max-margin solution.

\textbf{Implicit bias of neural network towards max-margin solution}. 
There is also a line of research on the implicit bias of neural network towards max-margin (hard margin SVM) solution under different settings and assumptions \citep{soudry2018implicit, ji2018risk, chizat2020implicit, wei2019regularization, nacson2019lexicographic, lyu2019gradient}. Our paper complements these research by establishing an exact equivalence between the infinite-width NN and SVM. Our equivalences not only include hard margin SVM but also include soft margin SVM and other $\ell_2$ regularized KMs.


\subsection{Neural Networks and Tangent Kernel}
We consider a general form of deep neural network $f$ with a linear output layer as \citep{liu2020linearity}. Let $[L] = \{1,...,L\}$, $\forall l \in [L]$, 
\begin{equation}\label{eq:NN}
    \begin{split}
    \alpha^{(0)}(w, x) = x, \
    \alpha^{(l)}(w, x) = \phi_l(w^{(l)}, \alpha^{(l-1)}), \
    f(w, x) = \frac{1}{\sqrt{m_{L}}}  
    \inner{  w^{(L+1)}, \alpha^{(L)}(w, x) }, 
  \end{split}
\end{equation}
where each vector-valued function $\phi_l(w^{(l)}, \cdot): \mathbb{R}^{m_{l-1}} \rightarrow \mathbb{R}^{m_{l}}$, with parameter $w^{(l)} \in \mathbb{R}^{p_{l}}$ ($p_l$ is the number of parameters), is considered as a layer of the network. This definition includes the standard fully connected, convolutional (CNN), and residual (ResNet) neural networks as special cases. For a fully connected ReLU NN, $\alpha^{(l)}(w, x) = \sigma(\frac{1}{\sqrt{m_{l-1}}} w^{(l)} \alpha^{(l-1)})$ with $w^{(l)} \in \mathbb{R}^{m_{l} \times m_{l-1}}$ and $\sigma(z) = \max(0, z)$. 

\textbf{Initialization and parameterization.} In this paper, we consider the NTK parameterization \citep{jacot2018neural}, under which the constancy of the tangent kernel has been initially observed. Specifically, the parameters, $w := \{w^{(1)}; w^{(2)}; \cdots; w^{(L)}; w^{(L+1)}\}$ are drawn i.i.d. from a standard Gaussian, $\mathcal{N}(0, 1)$, at initialization, denoted as $w_0$. The factor $1 /\sqrt{m_L}$ in the output layer is required by the NTK parameterization in order that the output $f$ is of order $O(1)$. While we only consider NTK parameterization here, the results should be able to extend to general parameterization of kernel regime \citep{yang2020feature}.

\begin{definition}[Tangent kernel \cite{jacot2018neural}]\label{tangent_k}
  The tangent kernel associated with function $f(w, x)$ at some parameter $w$ is $\hat{\Theta}(w; x, x') = \inner{ \nabla_w f(w, x), \nabla_w f(w, x') } $. Under certain conditions (usually infinite width limit and NTK parameterization), the tangent kernel at initialization converges in probability to a deterministic limit and keeps constant during training, $\hat{\Theta}(w; x, x') \rightarrow \Theta_{\infty}(x, x')$. This limiting kernel is called \textit{Neural Tangent Kernel (NTK)}. 
\end{definition}

\subsection{Kernel Machines}
Kernel machine (KM) is a model of the form $g(\beta, x) = \varphi  ( \inner{ \beta, \Phi(x)  }  + b )$, where $\beta$ is the model parameter and $\Phi$ is a mapping from input space to some feature space, $\Phi: \mathcal{X} \rightarrow \mathcal{F}$. $\varphi$ is an optional nonlinear function, such as identity mapping for kernel regression and $sign(\cdot)$ for SVM and logistic regression. The kernel can be exploited whenever the weight vector can be expressed as a linear combination of the training points, $\beta = \sum_{i = 1}^{n} \alpha_i \Phi(x_i) $ for some value of $\alpha_i$, $i \in [n]$, implying that we can express $g$ as $g(x) = \varphi (\sum_{i = 1}^{n}  \alpha_i K(x, x_i)  + b)$, where $K(x, x_i) = \inner{ \Phi(x), \Phi(x_i)  }$ is the kernel function. For a NN in NTK regime, we have $f(w_t, x) \approx f(w_0, x) +  \inner{ \nabla_{w} f(w_0, x), w_t - w_0 }$, which makes the NN linear in the gradient feature mapping $x \rightarrow \nabla_{w} f(w_0, x)$. Under squared loss, it is equivalent to kernel regression with $\Phi(x) = \nabla_{w} f(w_0, x)$ (or equivalently using NTK as the kernel), $\beta = w_t - w_0$ and $\varphi$ identity mapping \citep{arora2019exact}.

As far as we know, there is no work establishing the equivalence between fully trained NN and SVM. \cite{domingos2020every} showed that NNs trained by gradient flow are approximately "kernel machines" with the weights and bias as functions of input data, which however can be much more complex than a typical kernel machine. In this work, we compare the dynamics of SVM and NN trained by subgradient descent with soft margin loss and show the equivalence between them in the infinite width limit.

\subsection{Subgradient Optimization of Support Vector Machine}
We first formally define the standard soft margin SVM and then show how the subgradient descent can be applied to get an estimation of the SVM primal problem. For simplicity, we consider the homogenous model, $g(\beta, x) = \inner{ \beta, \Phi(x) }$.\footnote{Note one can always deal with the bias term $b$ by adding each sample with an additional dimension, $\Phi(x)^T \leftarrow [\Phi(x)^T, 1], \beta^T \leftarrow [\beta^T, 1]$.}  
\begin{definition}[Soft margin SVM]
  Given labeled samples $\{(x_i, y_i) \}_{i=1}^{n}$ with $y_i \in \{-1, +1\}$, the hyperplane $\beta^*$ that solves the below optimization problem realizes the soft margin classifier with geometric margin $ \gamma = 2 / \norm{\beta^*}$. 
  \begin{equation*}
    \begin{split}
      \min_{\beta, \xi} \ & \frac{1}{2} \norm{ \beta }^2 + C \sum_{i = 1}^{n} \xi _i, \quad
      \text{s.t.}\ y_i \inner{ \beta, \Phi(x_i) } \geq 1 - \xi_i, 
      \ \xi_i \geq 0, \ i \in [n],
    \end{split}
  \end{equation*}
\end{definition}



\begin{proposition}
  The above primal problem of soft margin SVM can be equivalently formulated as
  \begin{equation}\label{eq:loss_svm}
      \min_{\beta} \frac{1}{2}\norm{ \beta }^2 + C \sum_{i = 1}^{n} \max(0, 1 -  y_i \inner{ \beta, \Phi(x_i) } ),
  \end{equation}
  where the second term is a hinge loss. Denote this function as $L(\beta)$, which is strongly convex in $\beta$.
\end{proposition}
From this, we see that the SVM technique is equivalent to empirical risk minimization with $\ell_2$ regularization, where in this case the loss function is the nonsmooth hinge loss. The classical approaches usually consider the dual problem of SVM and solve it as a quadratic programming problem. Some recent algorithms, however, use subgradient descent \citep{shalev2011pegasos} to optimize Eq. (\ref{eq:loss_svm}), which shows significant advantages when dealing with large datasets.

In this paper, we consider the soft margin SVM trained by subgradient descent with $L(\beta)$. We use the subgradient $\nabla_{\beta} L(\beta) = \beta - C \sum_{i=1}^{n} \mathbbm{1}(y_i g(\beta, x_i) < 1)  y_i  \Phi(x_i)$, where $\mathbbm{1}(\cdot)$ is the indicator function. As proved in \citep{shalev2011pegasos}, we can find a solution of accuracy $\epsilon$, i.e. $L(\beta) - L(\beta^*) \leq \epsilon$, in $\tilde{O}(1/\epsilon)$ iterations. Other works also give convergence guarantees for subgradient descent of convex functions \citep{boyd2003subgradient, bertsekas2015convex}. In the following analysis, we will generally assume the convergence of SVM trained by subgradient descent.

\section{Main Theoretical Results}
In this section, we describe our main results. We first derive the continuous (gradient flow) and discrete dynamics of SVM trained by subgradient descent (in Section~\ref{sec:dynamics_svm}) and the dynamics of NN trained by soft margin loss (in Section~\ref{sec:soft_nn} and Section~\ref{sec:dynamics_nn}). We show that they have similar dynamics, characterized by an inhomogeneous linear differential (difference) equation, and have the same convergence rate under reasonable assumption. Next, we show that their dynamics are exactly the same in the infinite width limit because of the constancy of tangent kernel and thus establish the equivalence (Theorem~\ref{theorem:equivalence}). Furthermore, in Section~\ref{sec:general_loss}, we generalize our theory to general loss functions with $\ell_2$ regularization and establish the equivalences between NNs and a family of $\ell_2$ regularized KMs as summarized in Table~\ref{table:loss_kernel}.

\subsection{Dynamics of Soft Margin SVM}\label{sec:dynamics_svm}
For simpicity, we denote $\beta_t$ as $\beta$ at some time $t$ and $g_t(x) = g(\beta_t, x)$. The proofs of the following two theorems are detailed in Appendix \ref{app:Dy_svm}.
\begin{theorem}[Continuous dynamics and convergence rate of SVM]
  Consider training soft margin SVM by subgradient descent with infinite small learning rate (gradient flow \citep{ambrosio2008gradient}): $\frac{d \beta_t}{dt} = - \nabla_{\beta} L(\beta_t) $, the model $g_t(x)$ follows the below evolution:
  \begin{equation}\label{eq:dy_svm}
    \frac{d g_t(x)}{d t} = - g_t(x) +  C \sum_{i=1}^{n}  \mathbbm{1}(y_i g_t(x_i) < 1)  y_i  K(x, x_i), 
  \end{equation}
  and has a linear convergence rate:
  \begin{equation*}
    L(\beta_t) - L(\beta^*) \leq e^{-2t} \parentheses{L(\beta_0) - L(\beta^*)} .
  \end{equation*}
  Denote $Q(t) = C \sum_{i=1}^{n}  \mathbbm{1}(y_i g_t(x_i) < 1)  y_i  K(x, x_i) $, which changes over time until convergence. The model output $g_t(x)$ at some time $T$ is
  \begin{equation}\label{eq:sol_dy}
      g_T(x) = e^{-T} \biggl( g_0(x)  + \int_{0}^{T} Q(t) e^{t} \,dt \biggr), 
      \quad
      \lim_{T \rightarrow \infty} g_T(x) = C \sum_{i = 1}^{n}   \mathbbm{1}(y_i g_T(x_i) < 1)  y_i K(x, x_i).
  \end{equation}
\end{theorem}

The continuous dynamics of SVM is described by an inhomogeneous linear differential equation (Eq. (\ref{eq:dy_svm})), which gives an analytical solution. From Eq. (\ref{eq:sol_dy}), we can see that the influence of initial model $g_0(x)$ deceases as time $T \rightarrow \infty$ and disappears at last.

\begin{theorem}[Discrete dynamics of SVM]\label{theorem:discrete_dynamics}
  Let $\eta \in (0, 1)$ be the learning rate. The dynamics of subgradient descent is 
  \begin{equation}\label{eq:discrete_dynamics}
    g_{t+1}(x)  - g_t(x) = -\eta g_t(x) + \eta C \sum_{i=1}^{n}  \mathbbm{1}(y_i g_t(x_i) < 1)  y_i  K(x, x_i). 
  \end{equation}
  Denote $Q(t) = \eta C \sum_{i=1}^{n}  \mathbbm{1}(y_i g_t(x_i) < 1)  y_i  K(x, x_i) $, which changes over time. The model output $g_t(x)$ at some time $T$ is
  \begin{equation*}
    g_T(x) 
    = (1 - \eta)^{T}  \biggl( g_0(x)  + \sum_{t=0}^{T-1}  (1 - \eta)^{-t-1} Q(t) \biggr),
    \lim_{T \rightarrow \infty} g_T(x) = C  \sum_{i = 1}^{n}   \mathbbm{1}(y_i g_T(x_i) < 1)  y_i K(x, x_i).
  \end{equation*}
\end{theorem}
The discrete dynamics is characterized by an inhomogeneous linear difference equation (Eq. (\ref{eq:discrete_dynamics})). The discrete dynamics and solution of SVM have similar structures as the continuous case.

\subsection{Soft Margin Neural Network}\label{sec:soft_nn}
We first formally define the soft margin neural network and then derive the dynamics of training a neural network by subgradient descent with a soft margin loss. We will consider a neural network defined as Eq. (\ref{eq:NN}). For convenience, we redefine $f(w, x) = \inner{ W^{(L+1)}, \alpha^{(L)}(w, x)  } $ with $W^{(L+1)} = \frac{1}{\sqrt{m_{L}}} w^{(L+1)}$ and $w := \{w^{(1)}; w^{(2)}; \cdots; w^{(L)}; W^{(L+1)}\}$.

\begin{definition}[Soft margin neural network]
  Given samples $\{(x_i, y_i) \}_{i=1}^{n}$, $y_i \in \{-1, +1\}$, the neural network $w^*$ defined as Eq. (\ref{eq:NN}) that solves the following two equivalent optimization problems
  \begin{equation*}
      \min_{w, \xi} \frac{1}{2} \norm{ W^{(L+1)} }^2 + C \sum_{i = 1}^{n} \xi _i,  \quad 
      \text{s.t.} \ y_i f(w, x_i) \geq 1 - \xi_i,
            \ \xi_i \geq 0, \ i \in [n],
  \end{equation*}
  \begin{equation}\label{eq:soft_nn_loss}
      \min_{w} \frac{1}{2} \norm{ W^{(L+1)} }^2  + C \sum_{i = 1}^{n} \max(0, 1 -  y_i f(w, x_i) ),
  \end{equation}
  realizes the soft margin classifier with geometric margin $ \gamma = 2 / \norm{ W_*^{(L+1)} }$. Denote Eq. (\ref{eq:soft_nn_loss}) as $L(w)$ and call it \textit{soft margin loss}.
\end{definition}
This is generally a hard nonconvex optimization problem, but we can apply subgradient descent to optimize it heuristically. At initialization, $\norm{ W_0^{(L+1)} }^2 = O(1)$. The derivative of the regularization for $w^{(L+1)}$ is $ w^{(L+1)} / \sqrt{m_{L}} = O(1 / \sqrt{m_{L}}) \rightarrow 0$ as $m \rightarrow \infty$. For a fixed $\alpha^{(L)}(w, x) $, this problem is same as SVM with $\Phi(x) = \alpha^{(L)}(w, x) $, kernel $K(x, x') = \alpha^{(L)}(w, x)  \cdot \alpha^{(L)}(w, x') $ and parameter $\beta = W^{(L+1)}$. If we only train the last layer of an infinite-width soft margin NN, it corresponds to an SVM with a NNGP kernel \citep{lee2017deep, matthews2018gaussian}. But for a fully-trained NN, $\alpha^{(L)}(w, x) $ is changing over time.

\subsection{Dynamics of Neural Network Trained by Soft Margin Loss}\label{sec:dynamics_nn}
Denote the hinge loss in $L(w)$ as $L_h(y_i, f(w, x_i)) =  C \max(0, 1 -  y_i f(w, x_i) )$. We use the same subgradient as that for SVM, $L_h'(y_i, f(w, x_i)) = - C y_i  \mathbbm{1}(y_i f(w, x_i) < 1) $.

\begin{theorem}[Continuous dynamics and convergence rate of NN] 
  Suppose a NN $f(w, x)$ defined as Eq. (\ref{eq:NN}), with $f$ a differentiable function of $w$, is learned from a training set $\{(x_i, y_i)\}_{i=1}^{n}$ by subgradient descent with $L(w)$ and gradient flow. Then the network has the following dynamics:
  \begin{equation*}
    \frac{d f_t(x)}{d t} =  - f_t(x) + C \sum_{i = 1}^{n} \mathbbm{1}(y_i f_t(x_i) < 1)  y_i  \hat{\Theta}(w_t; x, x_i).
  \end{equation*}
   Let $\hat{\Theta}(w_t) \in \mathbb{R}^{n \times n}$ be the tangent kernel evaluated on the training set and $\lambda_{min}(\hat{\Theta}(w_t))$ be its minimum eigenvalue. Assume $\lambda_{min}(\hat{\Theta}(w_t)) \geq \frac{2}{C}$, then NN has at least a linear convergence rate, same as SVM:
  \begin{equation*}
    L(w_t) - L(w^*) \leq e^{-2t}\parentheses{L(w_0) - L(w^*)} .
\end{equation*}
\end{theorem}
The proof is in Appendix~\ref{app:dy_nn}. The key observation is that when deriving the dynamics of $f_t(x)$, the $\frac{1}{2} \norm{ W^{(L+1)} }^2$ term in the loss function will produce a $f_t(x)$ term and the hinge loss will produce the tangent kernel term, which overall gives a similar dynamics to that of SVM. Comparing with the previous continuous dynamics without regularization \citep{jacot2018neural, lee2019wide}, our result has an extra $- f_t(x)$ here because of the regularization term of the loss function. The convergence rate is proved based on a sufficient condition for the PL inequality. The assumption of $\lambda_{min}(\hat{\Theta}(w_t)) \geq \frac{2}{C}$ can be guaranteed in a parameter ball when $\lambda_{min}(\hat{\Theta}(w_0)) > \frac{2}{C}$, by using a sufficiently wide NN \citep{liu2020loss}.

If the tangent kernel $\hat{\Theta}(w_t; x, x_i)$ is fixed, $\hat{\Theta}(w_t; x, x_i) \rightarrow \hat{\Theta}(w_0; x, x_i)$, the dynamics of NN is the same as that of SVM (Eq. (\ref{eq:dy_svm})) with kernel $\hat{\Theta}(w_0; x, x_i)$, assuming the neural network and SVM have same initial output $g_0(x) = f_0(x)$.\footnote{This can be done by setting the initial values to be $0$, i.e. $g_0(x) = f_0(x) =0$.} And this consistency of tangent kernel is the case for infinitely wide neural networks of common architectures, which does not depend on the optimization algorithm and the choice of loss function, as discussed in \citep{liu2020linearity}.

\textbf{Assumptions.} We assume that (vector-valued) layer functions
$\phi_l(w, \alpha), l \in [L]$ are $L_{\phi}$-Lipschitz continuous and twice differentiable with respect to input $\alpha$ and parameters $w$. The assumptions serve for the following theorem to show the constancy of tangent kernel. 

\begin{theorem}[Equivalence between NN and SVM]\label{theorem:equivalence}
  As the minimum width of the NN, $m = \min_{l \in [L]} m_l$, goes to infinity, the tangent kernel tends to be constant, $\hat{\Theta}(w_t; x, x_i) \rightarrow \hat{\Theta}(w_0; x, x_i)$. Assume $g_0(x) = f_0(x)$. Then the infinitely wide NN trained by subgradient descent with soft margin loss has the same dynamics as SVM with $\hat{\Theta}(w_0; x, x_i)$ trained by subgradient descent:
  \begin{equation*}
    \frac{d f_t(x)}{d t} =  - f_t(x) + C \sum_{i = 1}^{n} \mathbbm{1}(y_i f_t(x_i) < 1)  y_i  \hat{\Theta}(w_0; x, x_i).
  \end{equation*}
  And thus such NN and SVM converge to the same solution.
\end{theorem}

The proof is in Appendix~\ref{app:equivalence}. We apply the results of \citep{liu2020linearity} to show the constancy of tangent kernel in the infinite width limit. Then it is easy to check the dynamics of infinitely wide NN and SVM with NTK are the same. We give finite-width bounds for general loss functions in the next section. This theorem establishes the equivalence between infinitely wide NN and SVM for the first time. Previous theoretical results of SVM \citep{cristianini2000introduction, shawe2004kernel, scholkopf2002learning, steinwart2008support} can be directly applied to understand the generalization of NN trained by soft margin loss. Given the tangent kernel is constant or equivalently the model is linear, we can also give the discrete dynamics of NN (Appendix~\ref{app:dis_nn}), which is identical to that of SVM. Compared with the previous discrete-time gradient descent \citep{lee2019wide, yang2020feature}, our result has an extra $-\eta f_t(x)$ term because of the regularization term of the loss function. 
\begin{equation*}
    f_{t+1}(x)  - f_t(x) = -\eta f_t(x) + \eta C \sum_{i=1}^{n}  \mathbbm{1}(y_i f_t(x_i) < 1)  y_i  \hat{\Theta}(w_0; x, x_i) .
\end{equation*}


\begin{table}
  \caption{Summary of our theoretical results on the equivalences between infinite-width NNs and a family of KMs. Thanks to the representer theorem \citep{scholkopf2002learning}, our $\ell_2$ regularized KMs can all apply kernel trick, meaning infinite NTK can be applied in these $\ell_2$ regularized KMs.}
  \label{table:loss_kernel}
  \small
  \centering
  \begin{tabular}{llll}
    \toprule
    $\lambda$ &Loss $l(z, y_i)$ &Kernel machine  &\\
    \midrule
    
    $\lambda=0$ (\citep{jacot2018neural, arora2019exact})   &$(y_i - z)^2$  &Kernel regression           &  \\
    
    
    \midrule
    $\lambda \rightarrow 0$ (ours)  &$\max(0, 1-y_iz)$  &Hard margin SVM           &  \\ \cmidrule{2-4}

    \multirow{5}{*}{$\lambda> 0$ (ours)}  &$\max(0, 1-y_iz)$   &(1-norm) soft margin SVM &\\
    &$\max(0, 1-y_iz)^2$  &2-norm soft margin SVM        &  \\
    &$\max(0, \abs{ y_i - z}  - \epsilon )$  &Support vector regression           &   \\
    &$(y_i - z)^2$  &Kernel ridge regression (KRR)          &  \\
    &$\log(1+e^{-y_i z})$  &Logistic regression with $\ell_2$ regularization     & \\
    \bottomrule
  \end{tabular}
\end{table}

\subsection{General Loss Functions}\label{sec:general_loss}
We note that above analysis does not have specific dependence on the hinge loss. Thus we can generalize our analysis to general loss functions $l(z, y_i)$, where $z$ is the model output, as long as the loss function is differentiable (or has subgradients) with respect to $z$, such as squared loss and logistic loss. Besides, we can scale the regularization term by a factor $\lambda$ instead of scaling $l(z, y_i)$ with $C$ as it for SVM, which are equivalent. Suppose the loss function for the KM and NN are
\begin{gather}
    L(\beta) = \frac{\lambda}{2}\norm{ \beta }^2 + \sum_{i = 1}^{n} l(g(\beta, x_i), y_i),\label{eq:loss_svm2} \quad
    L(w) = \frac{\lambda}{2}\norm{ W^{(L+1)} }^2 + \sum_{i = 1}^{n} l(f(w, x_i), y_i) .
\end{gather}
Then the continuous dynamics of $g_t(x)$ and $f_t(x)$ are
\begin{gather}
    \frac{d g_t(x)}{d t} =  - \lambda g_t(x) - \sum_{i = 1}^{n} l'(g_t(x_i), y_i) K(x, x_i), \\
  \frac{d f_t(x)}{d t} =  - \lambda f_t(x) - \sum_{i = 1}^{n} l'(f_t(x_i), y_i) \hat{\Theta}(w_t; x, x_i), \label{eq:dy_general}
\end{gather}
where $l'(z, y_i) = \frac{\partial l(z, y_i)}{\partial z}$. In the situation of $\hat{\Theta}(w_t; x, x_i) \rightarrow \hat{\Theta}(w_0; x, x_i)$  and $K(x, x_i) = \hat{\Theta}(w_0; x, x_i)$, these two dynamics are the same (assuming $g_0(x) = f_0(x)$). When $\lambda=0$, we recover the previous results of kernel regression. When $\lambda > 0$, we have our new results of $\ell_2$ regularized loss functions. Table~\ref{table:loss_kernel} lists the different loss functions and the corresponding KMs that infinite-width NNs are equivalent to. KRR is considered in \citep{hu2019simple} to analyze the generalization of NN. However, they directly assume NN as a linear model and use it in KRR. Below we give finite-width bounds on the difference between the outputs of NN and the corresponding KM. The proof is in Appendix~\ref{app:bound_output}.

\begin{theorem}[Bounds on the difference between NN and KM]
Assume $g_0(x) = f_0(x), \forall x$ and $K(x, x_i) = \hat{\Theta}(w_0; x, x_i)$ \footnote{Linearized NN is a special case of such $g$.}. Suppose the KM and NN are trained with losses (\ref{eq:loss_svm2}) and gradient flow. Suppose $l$ is $\rho$-lipschitz and $\beta_l$-smooth for the first argument (i.e. the model output). Given any $w_T \in B(w_0; R) := \{ w: \norm{ w - w_0 } \leq R \}$ for some fixed $R > 0$, for training data $X \in \mathbb{R}^{d \times n}$ and a test point $x \in \mathbb{R}^{d}$, with high probability over the initialization,
\begin{gather*}
    \norm{f_T(X) - g_T(X)} = O(\frac{ e^{\beta_l \norm{\hat{\Theta}(w_0)}} R^{3L+1} \rho n^{\frac{3}{2}} \ln{m}}{\lambda \sqrt{m}}) , \\
    \norm{f_T(x) - g_T(x)} = O(\frac{ e^{\beta_l \norm{\hat{\Theta}(w_0; X, x)}} R^{3L+1} \rho n \ln{m}}{\lambda \sqrt{m}}) .
\end{gather*}
where $f_T(X), g_T(X) \in \mathbb{R}^{n}$ are the outputs of the training data and $\hat{\Theta}(w_0; X, x) \in \mathbb{R}^{n}$ is the tangent kernel evaluated between training data and test point.
\end{theorem}


\section{Discussion}
In this section, we give some extensions and applications of our theory. We first show that every finite-width neural network trained by a $\ell_2$ regularized loss function is approximately a KM in Section~\ref{sec:finite-width}, which enables us to compute non-vacuous generalization bound of NN via the corresponding KM. Next, in Section~\ref{sec:robustness}, we show that our theory of equivalence (in Section~\ref{sec:dynamics_nn}) is useful to evaluating the robustness of over-parameterized NNs with infinite width. In particular, our theory allows us to deliver nontrivial robustness certificates for infinite-width NNs, while existing robustness verification methods \citep{gowal2018effectiveness, weng2018towards, zhang2018efficient} would become much looser (decrease at a rate of $O(1 /\sqrt{m})$) as the width of NN increases and trivial with infinite width (the experiment results are in Section~\ref{sec:experiments} and Table~\ref{table:Robustness}).

\subsection{Finite-width Neural Network Trained by \(\ell_2\) Regularized Loss}\label{sec:finite-width}
Inspired by \citep{domingos2020every}, we can also show that every NN trained by (sub)gradient descent with loss function (\ref{eq:loss_svm2}) is approximately a KM without the assumption of infinite width.
\begin{theorem} \label{theorem:km}
  Suppose a NN $f(w, x)$, is learned from a training set $\{(x_i, y_i)\}_{i=1}^{n}$ by (sub)gradient descent with loss function (\ref{eq:loss_svm2}) and gradient flow. Assume $\text{sign}(l'(f_t(x_i), y_i)) = \text{sign}(l'(f_0(x_i), y_i)), \forall t \in \bracket{0, T}$.\footnote{This is the case for hinge loss and logistic loss.} Then at some time $T > 0$,
  \begin{equation*}
      f_T(x) = \sum_{i = 1}^{n} a_i K(x, x_i) + b, 
      \quad \text{with} \quad 
      K(x, x_i) = e^{-\lambda T} \int_{0}^{T}  \abs{l'(f_t(x_i), y_i)} \hat{\Theta}(w_t; x, x_i) e^{\lambda t} \,dt,
  \end{equation*}
  and $a_i = - \text{sign}(l'(f_0(x_i), y_i))$, $b = e^{-\lambda T} f_0(x)$.
\end{theorem}

See the proof in Appendix~\ref{app:kernel_machine}, which utilizes the solution of inhomogeneous linear differential equation instead of integrating both sides of the dynamics (Eq. (\ref{eq:dy_general})) directly \citep{domingos2020every}. Note in Theorem~\ref{theorem:km}, $a_i$ is deterministic and independent with $x$, different with \citep{domingos2020every} that has $a_i$ depends on $x$. Deterministic $a_i$ makes the function class simpler. Combing Theorem~\ref{theorem:km} with a bound of the Rademacher complexity of the KM \citep{bartlett2002rademacher} and a standard generalization bound using Rademacher complexity \citep{mohri2018foundations}, we can compute the generalization bound of NN via the corresponding KM. See Appendix~\ref{app:generalization} for more background and experiments. The generalization bound we get will depend on $a_i$, which depends on the label $y_i$. This differs from traditional complexity measures that cannot explain the random label phenomenon \cite{zhang2021understanding}. 

\begin{remark}
We note that the kernel here is valid only when $\abs{l'(f_t(x_i), y_i)}$ is a constant, e.g. $l'(f_t(x_i), y_i) = -y_i$ at the initial training stage of hinge loss with $f_0(x) = 0$, otherwise the kernel is not symmetric.
\end{remark}





\subsection{Robustness Verification of Infinite-width Neural Network}\label{sec:robustness}
Our theory of equivalence allows us to deliver nontrivial robustness certificates for infinite-width NNs by considering the equivalent KMs. For an input $x_0 \in \mathbb{R}^d$, the objective of robustness verification is to find the largest ball such that no examples within this ball $x \in B(x_0, \delta)$ can change the classification result. Without loss of generality, we assume $g(x_0) > 0$. The robustness verification problem can be formulated as follows, 
\begin{equation}\label{eq:robustness}
    \max \ \delta,  \quad \;
    \textrm{s.t.} \ \ g(x) > 0, \  \forall x \in B(x_0, \delta) . 
\end{equation}
For an infinitely wide two-layer fully connected ReLU NN, $f(x) = \frac{1}{\sqrt{m}}\sum_{j=1}^m v_j \sigma(\frac{1}{\sqrt{d}}w_j^T x)$, where $\sigma(z) = \max(0, z)$ is the ReLU activation, the NTK is
\begin{equation*}
    \Theta(x, x') 
    = \frac{\inner{x, x'}}{d}(\frac{\pi - \arccos(u)}{\pi}) + \frac{\norm{x}\norm{x'}}{2\pi d} \sqrt{1-u^2} ,
\end{equation*}
where $u = \frac{ \inner{x, x'}}{\norm{x}\norm{x'}} \in \left[-1, 1\right]$. See the proof of the following theorem in Appendix~\ref{app:robust}.

\begin{theorem}
  Consider the $\ell_{\infty}$ perturbation, for $x \in B_{\infty}(x_0, \delta) = \{x \in \mathbb{R}^d: \norm{x - x_0}_{\infty} \leq \delta \}$, we can bound $\Theta(x, x') $ into some interval $[\Theta^L(x, x'), \Theta^U(x, x')]$. Suppose $g(x) = \sum_{i = 1}^{n}  \alpha_i \Theta(x, x_i) $, where $\alpha_i$ are known after solving the KM problems (e.g. SVM and KRR). Then we can lower bound $g(x)$ as follows,
  \begin{equation*}
    g(x) \geq 
    \sum_{i = 1, \alpha_i >0}^{n}  \alpha_i \Theta^L(x, x_i) + \sum_{i = 1, \alpha_i<0}^{n}  \alpha_i \Theta^U(x, x_i) . 
\end{equation*}
\end{theorem}
Using a simple binary search and above theorem, we can find a certified lower bound for (\ref{eq:robustness}). Because of the equivalence between the infinite-width NN and KM, the certified lower bound we get for the KM is equivalently a certified lower bound for the corresponding infinite-width NN.

\begin{figure}
  \centering
  \includegraphics[width=\textwidth]{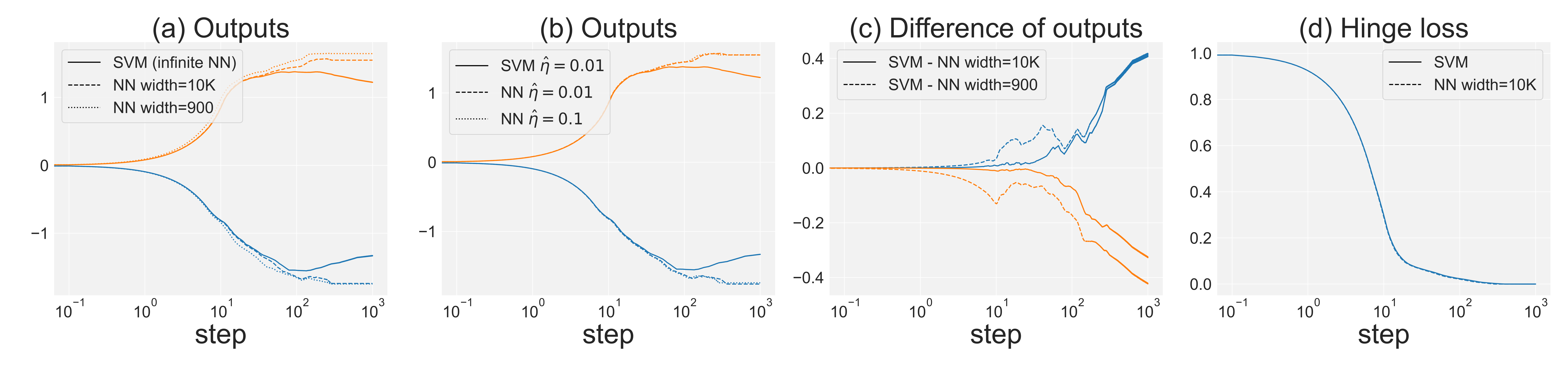}
  \caption{Training dynamics of neural network and SVM behave similarly. (a)(b) show dynamics of outputs for randomly selected two samples. (c) shows the difference between the outputs of SVM and NN. The dynamics of SVM agrees better with wider NN. (d) shows the dynamics of hinge loss for SVM and NN. Without specification, the width of NN is 10K and $\hat{\eta}=0.1$.}
  \label{fig:gd}
\end{figure}

\section{Experiments}\label{sec:experiments}
\textbf{(\uppercase\expandafter{\romannumeral1}) Verification of the equivalence.} The first experiment verifies the equivalence between soft margin SVM with NTK trained by subgradient descent and the infinite-width NN trained by soft margin loss. We train the SVM and 3-layer fully connected ReLU NN for a binary MNIST \citep{lecun1998gradient} classification ($0$ and $1$) with learning rate $\hat{\eta} = 0.1$ and $\hat{\eta} = 0.01$ with full batch subgradient descent on $n=128$ samples, where $\hat{\eta}$ is the learning rate used in experiments. Figure~\ref{fig:gd} shows the dynamics of the outputs and loss for NN and SVM. Since the regularization terms in the loss of NN and SVM are different, we just plot the hinge loss. It can be seen that the dynamics of NN and SVM agree very well. We also do a stochastic subgradient descent case for binary classification on full MNIST $0$ and $1$ data ($12665$ training and $2115$ test) with learning rate $\hat{\eta} = 1$ and batch size $64$, shown in Figure~\ref{fig:SGD}. For more details, please see Appendix~\ref{app:experiment}. 


\textbf{(\uppercase\expandafter{\romannumeral2}) Robustness of over-parameterized neural network.} Table~\ref{table:Robustness} shows the robustness certificates of two-layer overparameterized NNs with increasing width and SVM (which is equivalent to infinite-width two-layer ReLU NN) on binary classification of MNIST (0 and 1). We use the NN robustness verification algorithm (IBP) \citep{gowal2018effectiveness} to compute the robustness certificates for two-layer overparameterized NNs. The robustness certificate for SVM is computed using our method in Section~\ref{sec:robustness}. As demonstrated in Table~\ref{table:Robustness}, the certificate of NN almost decrease at a rate of $O(1 / \sqrt{m})$ and will decrease to 0 as $m \rightarrow \infty$, where $m$ is the width of the hidden layer. We show that this is due to the bound propagation in Appendix~\ref{app:ibp}. Unfortunately, the decrease rate will be faster if the NN is deeper. The same problem will happen for LeCun initialization as well, which is used in PyTorch for fully connected layers by default. Notably, however, thanks to our theory, we could compute \textit{nontrivial} robustness certificate for an infinite-width NN through the equivalent SVM as demonstrated. 


\begin{table}[h!]
  \caption{Robustness certified lower bounds of two-layer ReLU NN and SVM (infinite-width two-layer ReLU NN) tested on binary classification of MNIST (0 and 1). 100 test: randomly selected 100 test samples. Full test: full test data. Test only on data that classified correctly. std is computed over data samples. All models have test accuracy 99.95\%. All values are mean of 5 experiments.}
  \label{table:Robustness}
  \small 
  \centering
  \begin{tabular}{lllll}
    \toprule
    &&\multicolumn{2}{c}{Robustness certificate $\delta$ (mean $\pm$ std) $\times10^{-3}$}  \\
    \cmidrule(r){3-4}
    Model    &Width      & 100 test  &Full test     \\
    \midrule
    
    NN     & $10^3$      
    &7.4485 $\pm$ 2.5667  
    &7.2708 $\pm$ 2.1427    \\
    
    NN     & $10^4$      
    &2.9861 $\pm$ 1.0730   
    &2.9367 $\pm$ 0.89807   \\
    
    NN     & $10^5$     
    &0.99098 $\pm$ 0.35775   
    &0.97410 $\pm$ 0.29997   \\
    
    NN     & $10^6$     
    &0.31539 $\pm$ 0.11380   
    &0.30997 $\pm$ 0.095467   \\
    \midrule
    
    SVM     &$\infty$     
    &8.0541 $\pm$ 2.5827   
    &7.9733 $\pm$ 2.1396   \\
    
    \bottomrule
  \end{tabular}
\end{table}

\textbf{(\uppercase\expandafter{\romannumeral3}) Comparison with kernel regression.} Table~\ref{table:Robustness_kernel_machines} compares our $\ell_2$ regularized models (KRR and SVM with NTK) with the previous kernel regression model ($\lambda=0$ for KRR). All the robustness certified lower bounds are computed using our method in Section~\ref{sec:robustness}. While the accuracies of different models are similar, as the regularization increases, the robustness of KRR increases. The robustness of SVM outperforms the KRR with same regularization magnitude a lot. Our theory enables us to train an equivalent infinite-width NN through SVM and KRR, which is intrinsically more robust than the previous kernel regression model.

\begin{table}[h!]
  \caption{Robustness of equivalent infinite-width NN models with different loss functions (see Table~\ref{table:loss_kernel}) on binary classification of MNIST (0 and 1). $\lambda$ is the parameter in Eq. (\ref{eq:loss_svm2}). }  
  \label{table:Robustness_kernel_machines}
  \small
  \centering
  \begin{tabular}{llllll}
    \toprule
    &Model   &$\lambda$  &Test accuracy    &Robustness certificate $\delta$     &Robustness improvement  \\
    \midrule
    $\lambda=0$ (\citep{jacot2018neural, arora2019exact})  &KRR     &0      &99.95\%     &3.30202$\times10^{-5}$         & -\\
    \midrule
    \multirow{8}{*}{$\lambda> 0$ (ours)}
    &KRR     &0.001  &99.95\%     &3.756122$\times10^{-5}$   &1.14X\\
    &KRR     &0.01   &99.95\%     &6.505500$\times10^{-5}$   &1.97X\\
    &KRR     &0.1    &99.95\%     &2.229960$\times10^{-4}$   &6.75X\\
    &KRR     &1      &99.95\%     &0.001005                  &30.43X\\
    &KRR     &10     &99.91\%     &0.005181                  &156.90X\\
    &KRR     &100    &99.86\%     &0.020456                  &619.50X\\
    &KRR     &1000   &99.76\%     &0.026088                  &790.06X\\
    &SVM     &0.064  &99.95\%     &0.008054                  &243.91X\\
    \bottomrule
  \end{tabular}
\end{table}

\section{Conclusion and Future Works}\label{sec:conculsion}
In this paper, we establish the equivalence between SVM with NTK and the NN trained by soft margin loss with subgradient descent in the infinite width limit, and we show that they have the same dynamics and solution. We also extend our analysis to general $\ell_2$ regularized loss functions and show every neural network trained by such loss functions is approximately a KM. Finally, we demonstrate our theory is useful to compute \textit{non-vacuous} generalization bound for NN, \textit{non-trivial} robustness certificate for infinite-width NN while existing neural network robustness verification algorithm cannot handle, and with our theory, the resulting infinite-width NN from our $\ell_2$ regularized models is intrinsically more robust than that from the previous NTK kernel regression. For future research, since the equivalence between NN and SVM (and other $\ell_2$ regularized KMs) with NTK has been established, it would be very interesting to understand the generalization and robustness of NN from the perspective of these KMs. Our main results are currently still limited in the linear regime. It would be interesting to extend the results to the mean field setting or consider its connection with the implicit bias of NN.

\section{Acknowledgement}
We thank the anonymous reviewers for useful suggestions to improve the paper. We thank Libin Zhu for helpful discussions. We thank San Diego Supercomputer Center and MIT-IBM Watson AI Lab for computing resources. This work used the Extreme Science and Engineering Discovery Environment (XSEDE) \cite{xsede}, which is supported by National Science Foundation grant number ACI-1548562. This work used the Extreme Science and Engineering Discovery Environment (XSEDE) \textit{Expanse} at San Diego Supercomputer Center through allocation TG-ASC150024 and ddp390. T.-W. Weng is partially supported by National Science Foundation under Grant No. 2107189.



\newpage
\bibliographystyle{abbrvnat}
\bibliography{references}

\newpage
\appendix
\appendixpage

\renewcommand\thefigure{\thesection.\arabic{figure}}  

\renewcommand{\abs}[1]{\left\lvert#1\right\rvert }
\renewcommand{\norm}[1]{\left\lVert#1\right\rVert}
\renewcommand{\inner}[1]{\left\langle#1\right\rangle}

\section{Experiment Details}\label{app:experiment}
\setcounter{figure}{0}

\begin{figure}[h]
  \centering
  \includegraphics[width=\textwidth]{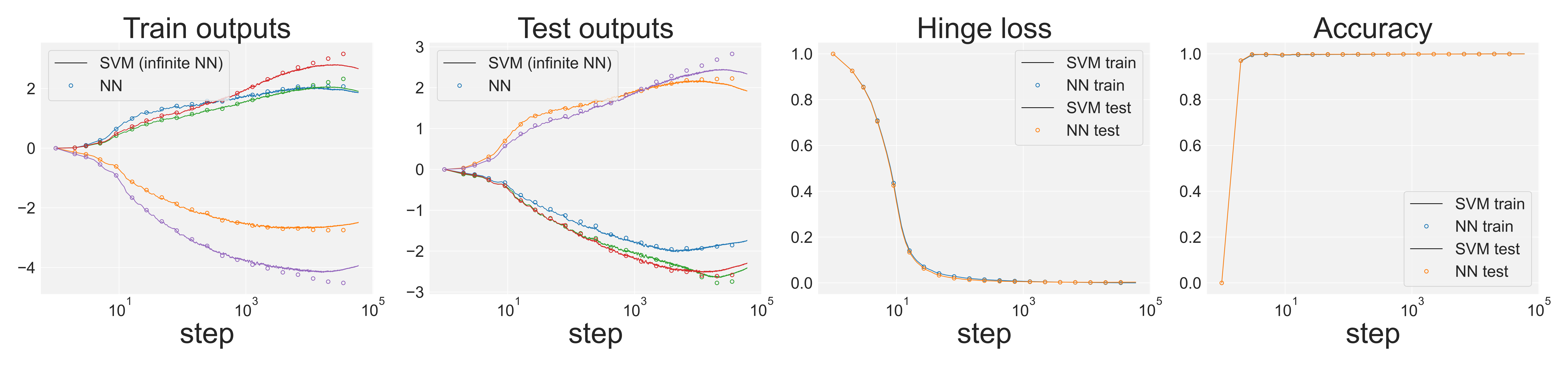}
  \caption{SVM and NN trained by stochastic subgradient descent for binary MNIST classification task on full $0$ and $1$ data with learning rate $\hat{\eta} = 1$ and batch size $64$. The width of NN is $10$K.}
  \label{fig:SGD}
\end{figure}

\subsection{SVM Training}
We use the following loss to train the SVM,
\begin{equation}
    L(\beta) = \frac{\lambda}{2}\norm{ \beta }^2 +  \frac{1}{n} \sum_{i = 1}^{n} \max(0, 1 -  y_i \inner{ \beta, \Phi(x_i) } ) .
\end{equation}
Let $\hat{\eta}$ be the learning rate for this loss in experiments. Then the dynamics of subgradient descent is 
\begin{equation}\label{eq:dy_dis}
    g_{t+1}(x) = (1 - \hat{\eta} \lambda) g_t(x) + \frac{\hat{\eta}}{n} \sum_{i=1}^{n}  \mathbbm{1}(y_i g_t(x_i) < 1)  y_i  K(x, x_i)  .
\end{equation}
Denote $Q(t) = \frac{\hat{\eta}}{n} \sum_{i=1}^{n}  \mathbbm{1}(y_i g_t(x_i) < 1)  y_i  K(x, x_i) $, which is a linear combination of $K(x, x_i)$ and changes over time. The model output $g_t(x)$ at some time $T$ is
\begin{equation}
    g_T(x) 
    = (1 - \hat{\eta} \lambda)^{T}  \biggl( g_0(x)  + \frac{\hat{\eta}}{n} \sum_{t=0}^{T-1}  (1 - \hat{\eta} \lambda)^{-t-1} Q(t) \biggr) .
\end{equation}
If we set $g_0(x) = 0$, we have
\begin{equation}
\begin{split}
    g_T(x)
    &= \sum_{t=0}^{T-1}  (1 - \hat{\eta} \lambda)^{T-1-t} Q(t) .
\end{split}
\end{equation}
We see that $g_T(x)$ is always a linear combination of kernel values $K(x, x_i)$ for $i=1, \dots, n$. Since $K(x, x_i)$ are fixed, we just need to store and update the weights of the kernel values. Let $\alpha_t \in \mathbb{R}^n$ be the weights at time $t$, that is 
\begin{equation}
    g_t(x) = \sum_{i=1}^{n} \alpha_{t,i} K(x, x_i) .
\end{equation}
Then according to Eq. (\ref{eq:dy_dis}), we update $\alpha$ at each subgradient descent step as follows.
\begin{equation}
    \alpha_{t+1, i} = (1 - \hat{\eta} \lambda) \alpha_{t, i} + \frac{\hat{\eta}}{n}  \mathbbm{1}(y_i g_t(x_i) < 1)  y_i, \quad \forall i \in \{1, \dots, n\} .
\end{equation}
For the SGD case, we sample $S_t \subseteq \{1, \dots, n\}$ at step $t$ and update the weights of this subset while keeping the other weights unchanged. 
\begin{gather*}
    \alpha_{t+1, i} = (1 - \hat{\eta} \lambda) \alpha_{t, i} + \frac{\hat{\eta}}{\abs{S_t}}  \mathbbm{1}(y_i g_t(x_i) < 1)  y_i, \quad \forall i \in S_t ,\\
    \alpha_{t+1, i} = \alpha_{t, i} , \quad \forall i \notin S_t .
\end{gather*}

The kernelized implementation of Pegasos \citep{shalev2011pegasos} set $\hat{\eta}_t = \frac{1}{\lambda t}$ for proving the convergence of the algorithm. In our experiments, we use constant $\hat{\eta}$. 


\subsection{More Details}
\textbf{(\uppercase\expandafter{\romannumeral1}) Verification of the equivalence.} The first experiment illustrates the equivalence between soft margin SVM with NTK trained by subgradient descent and NN trained by soft margin loss. We initialize 3-layer fully connected ReLU neural networks of width $10000$ and $900$, with NTK parameterization and make sure $f_0(x)=0$ by subtracting the initial values from NN's outputs. We initialize the parameter of SVM with $\beta_0 = 0$, and this automatically makes sure $g_0(x)=0$. SVM is trained by directly update the weights of kernel values \citep{shalev2011pegasos} and more details can be found in Appendix~\ref{app:experiment}. We set the regularization parameter as $\lambda=0.001$ and take the average of the hinge loss instead of sum.\footnote{This is equivalent to use $\lambda=0.001 \times n$ in Eq. (\ref{eq:loss_svm2}).} We train the NN and SVM for a binary MNIST \citep{lecun1998gradient} classification task ($0$ and $1$) with learning rate $\hat{\eta} = 0.1$ and $\hat{\eta} = 0.01$ with full batch subgradient descent on $n=128$ samples, where $\hat{\eta}$ is the learning rate used in experiments (see Appendix~\ref{app:experiment}). Figure~\ref{fig:gd} shows the dynamics of the outputs and loss for NN and SVM. Since the regularization term in the loss of NN and SVM are different, we just plot the hinge loss. We see the dynamics of NN and SVM agree well. We also do a stochastic subgradient descent case for binary MNIST classification task on full $0$ and $1$ data ($12665$ train data and $2115$ test data) with learning rate $\hat{\eta} = 1$ and batch size $64$, shown in Figure~\ref{fig:SGD}.

Experiments are implemented with PyTorch \citep{paszke2019pytorch} and the NTK of infinite-width NN is computed using Neural Tangents \citep{neuraltangents2020}. We do our experiments on 16G V100 GPU.

\section{Computing Non-vacuous Generalization Bounds via Corresponding Kernel Machines}\label{app:generalization}
\begin{figure}[h]
  \centering
  \includegraphics[width=\textwidth]{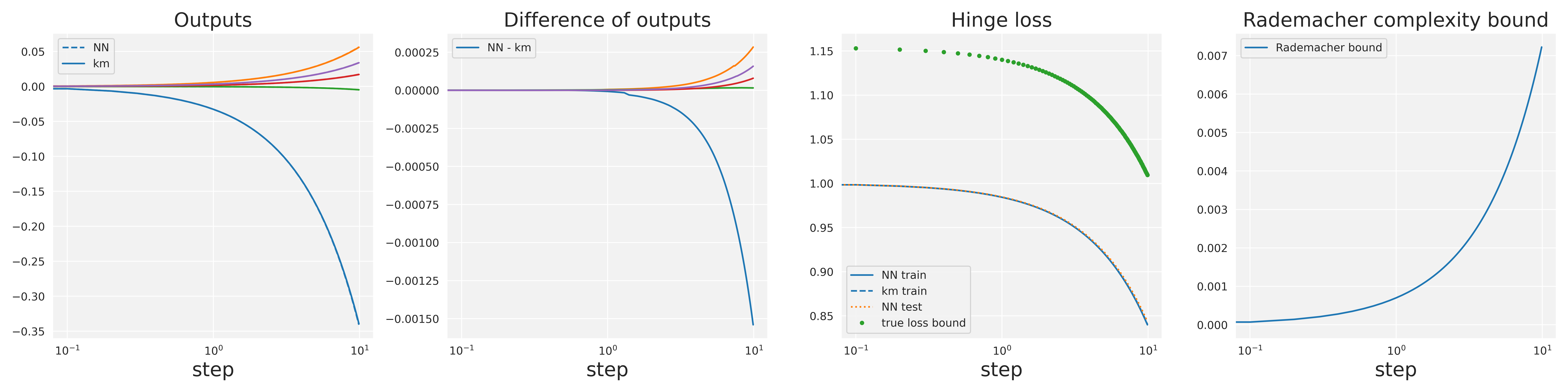}
  \caption{Computing non-vacuous generalization bounds via corresponding kernel machines. Two-layer NN with 100 hidden nodes trained by full-batch subgradient descent for binary MNIST classification task on $0$ and $1$ data of size 1000 with learning rate $\hat{\eta} = 0.1$ and hinge loss. We set $f_0(x)=0$. The kernel machine (KM) approximates NN very well. And we get a tight upper bound of the true loss by computing its Rademacher complexity. The confidence parameter is set as $1-\delta = 0.99$. Since the kernel is only valid when the loss gradient of the output is a constant, we only did the initial training stage of the hinge loss where the kernel is valid.}
  \label{fig:generalization}
\end{figure}
Using Theorem~\ref{theorem:km}, we can numerically compute the kernel machine that the NN is equivalent to, i.e. we can compute the kernel matrix and the weights at any time during the training. Then one can apply a generalization bound of kernel machines to give a generalization bound for this kernel machine (equivalently for this NN). Let $\mathcal{H}$ be the reproducing kernel Hilbert space (RKHS) corresponding to the kernel $K(\cdot, \cdot)$. The RKHS norm of a function $f(x) = \sum_{i=1}^n a_i K(x, x_i)$ is \footnote{Assume $f_0(x) = 0.$}
\begin{equation*}
    \norm{f}_{\mathcal{H}} = \norm{\sum_{i=1}^n a_i \Phi(x_i)} = \sqrt{\sum_{i=1}^n \sum_{j=1}^n a_i a_j K(x_i, x_j)}
\end{equation*}

\begin{lemma}[Lemma 22 in \cite{bartlett2002rademacher}]
  For a function class $\mathcal{F}_B = \{f(x) = \sum_{i=1}^n a_i K(x, x_i): \norm{f}_{\mathcal{H}} \leq B \} \subseteq \{x \rightarrow \inner{\beta, \Phi(x)}: \norm{\beta} \leq B \} $, its empirical Rademacher complexity can be bounded as
\begin{equation*}
    \hat{\mathcal{R}}_S(\mathcal{F}_B) 
    = \frac{1}{n} \expect_{\sigma_i \sim \{\pm 1\}^n} \bracket{\sup_{f \in \mathcal{F}_B} \sum_{i=1}^n \sigma_i f(x_i)}
    \leq \frac{B}{n} \sqrt{\sum_{i=1}^n K(x_i, x_i)}
\end{equation*}
\end{lemma}

Assume the data is sampled i.i.d. from some distribution $D$ and the population loss is $L_D(f) = \expect_{(x, y) \sim D}{\bracket{l(f(x), y)}}$. The empirical loss is $L_S(f) = \frac{1}{n} \sum_{i = 1}^{n} l(f(x_i), y_i)$. Combing with a standard generalization bound using Rademacher complexity blow \citep{mohri2018foundations}, we can get a bound of the population loss $L_D(f)$ for the kernel machine (equivalently for this NN).
\begin{lemma}
   Suppose the loss $\ell(\cdot, \cdot)$ is bounded in $\bracket{0, c}$, and is $\rho$-Lipschitz in the first argument. Then with probability at least $1-\delta$ over the sample $S$ of size $n$,
\begin{equation*}   
    \sup_{f \in \mathcal{F}}\{ L_D(f) - L_S(f) \} \leq 2 \rho \hat{\mathcal{R}}_S(\mathcal{F}) + 3c \sqrt{\frac{\log(2/\delta)}{2n}}
\end{equation*}

\end{lemma}

Most of the existing generalization bounds of NN \citep{bartlett2019nearly, long2019generalization} are vacacous since they have a dependence on the number of parameters. Compared to those, the bound for kernel machines does not have a dependence on the number of NN’s parameters, making it non-vacuous and promising.

\section{Dynamics of Support Vector Machine}\label{app:Dy_svm}
In this section, we derive the continuous and discrete dynamics of soft margin SVM trained by subgradient with the following loss function
\begin{equation}
    L(\beta) = \frac{1}{2}\norm{ \beta }^2 + C \sum_{i = 1}^{n} \max(0, 1 -  y_i \inner{ \beta, \Phi(x_i) } ),
\end{equation}
and the subgradient
\begin{equation}
  \nabla_{\beta} L(\beta_t) = \beta_t - C \sum_{i=1}^{n}  \mathbbm{1}(y_i g_t(x_i) < 1)  y_i  \Phi(x_i) .
\end{equation}

\begin{lemma}\label{lemma:PL}
$L(\beta)$ satisfies the Polyak- Lojasiewicz (PL) inequality,
\begin{equation}
    L(\beta_t) - L(\beta^*) \leq \frac{1}{2} \norm{\nabla_{\beta} L(\beta_t)}^2  \quad \forall \ \beta_t .
\end{equation}
where $\beta^* = \argmin_{\beta} L(\beta)$.
\end{lemma}
\begin{proof}
Since $L(\beta)$ is 1-strongly convex, by the definition of strong convexity and subgradient
\begin{equation}
    L(\beta) 
    \geq L(\beta_t) +  \inner{\nabla_{\beta} L(\beta_t), \beta - \beta_t} + \frac{1}{2} \norm{\beta - \beta_t}^2
\end{equation}
The right hand side is a convex quadratic function of $\beta$ (for fixed $\beta_t$). Setting the gradient with respect to $\beta$ equal to 0, we find that $\tilde{\beta} = \beta_t - \nabla_{\beta} L(\beta_t) $ minimize right hand side. Therefore we have
\begin{equation}
\begin{split}
    L(\beta) 
    &\geq L(\beta_t) +  \inner{\nabla_{\beta} L(\beta_t), \beta - \beta_t} + \frac{1}{2} \norm{\beta - \beta_t}^2 \\
    &\geq L(\beta_t) +  \inner{\nabla_{\beta} L(\beta_t), \tilde{\beta}  - \beta_t}  + \frac{1}{2} \norm{\tilde{\beta} - \beta_t}^2 \\
    &= L(\beta_t) - \frac{1}{2} \norm{\nabla_{\beta} L(\beta_t)}^2 .
\end{split}
\end{equation}
Since this holds for any $\beta$, we have
\begin{equation}
    L(\beta^*) \geq L(\beta_t) - \frac{1}{2} \norm{\nabla_{\beta} L(\beta_t)}^2 .
\end{equation}

\end{proof}

\subsection{Continuous Dynamics of SVM}
Here we give the detailed derivation of the dynamics of soft margin SVM trained by subgradient. In the learning rate $\eta \rightarrow 0$ limit, the subgradient descent equation, which can also be written as 
\begin{equation}
    \frac{\beta_{t+1} - \beta_t}{\eta} = - \nabla_{\beta} L(\beta_t) , 
\end{equation}
becomes a differential equation
\begin{equation}
  \frac{d \beta_t}{dt} = - \nabla_{\beta} L(\beta_t) . 
\end{equation}
This is known as gradient flow \citep{ambrosio2008gradient}. And we have defined the subgradient as 
\begin{equation}
  \nabla_{\beta} L(\beta_t) = \beta_t - C \sum_{i=1}^{n}  \mathbbm{1}(y_i g_t(x_i) < 1)  y_i  \Phi(x_i) .
\end{equation}
Applying the chain rule, the dynamics of $g_t(x) = \inner{\beta_t, \Phi(x)}$ is
\begin{equation}
  \begin{split}
    \frac{d g_t(x)}{d t} 
    &= \frac{\partial g_t(x)}{\partial \beta_t} \frac{d \beta_t}{dt} \\
    &= \inner{\Phi(x), - \beta_t + C \sum_{i=1}^{n}  \mathbbm{1}(y_i g_t(x_i) < 1)  y_i  \Phi(x_i) } \\
    &= - g_t(x) +  C \sum_{i=1}^{n}  \mathbbm{1}(y_i g_t(x_i) < 1)  y_i  K(x, x_i) .
  \end{split}
\end{equation}
Denoting $Q(t) = C \sum_{i=1}^{n}  \mathbbm{1}(y_i g_t(x_i) < 1) y_i  K(x, x_i) $, the equation becomes
\begin{equation}
    \frac{d g_t(x)}{d t} + g_t(x) = Q(t) .
\end{equation}
Note this is a first-order inhomogeneous differential equation. The general solution at some time $T$ is given by
\begin{equation}
    g_T(x) = e^{-T} \biggl( g_0(x)  + \int_{0}^{T} Q(t) e^{t} \,dt \biggr) .
\end{equation}

As we already know that the loss function is strongly convex, $\beta$ will converge to the global optimizer in this infinite small learning rate setting. This can be seen by
\begin{equation}
    \frac{d \parentheses{L(\beta_t) - L(\beta^*)} }{d t} 
    = \frac{d L(\beta_t) }{d t} 
    = \frac{\partial L(\beta_t) }{\partial \beta_t} \frac{d \beta_t }{d t} 
    = \inner{\nabla_{\beta} L(\beta_t), - \nabla_{\beta} L(\beta_t) }
    = - \norm{\nabla_{\beta} L(\beta_t) }^2 .
\end{equation}
We see that $L(\beta_t)$ is always decreasing. Since $L(\beta)$ is strongly convex and thus bounded from below, by monotone convergence theorem, $L(\beta_t)$ will always converge. By Lemma~\ref{lemma:PL}, we have the Polyak- Lojasiewicz (PL) inequality,
\begin{equation}
    L(\beta_t) - L(\beta^*) \leq \frac{1}{2} \norm{\nabla_{\beta} L(\beta_t)}^2
\end{equation}
Combining with above, we have 
\begin{equation}
    \frac{d \parentheses{L(\beta_t) - L(\beta^*)} }{d t} \leq -2 \parentheses{L(\beta_t) - L(\beta^*)} .
\end{equation}
Solving the equation, we get
\begin{equation}
    L(\beta_t) - L(\beta^*) \leq e^{-2t} \parentheses{L(\beta_0) - L(\beta^*)} .
\end{equation}
Thus we have a linear convergence rate.

Now, let us assume $g_T(x)$ will converge and see what is $g_T(x)$ as $T \rightarrow \infty$. As time increases $T \rightarrow \infty$, $e^{-T} g_0(x) \rightarrow 0$.
\begin{equation}
    g_T(x) \rightarrow e^{-T}  \int_{0}^{T} Q(t) e^{t} \,dt 
\end{equation}
$Q(t)$ is changing over time due to $g_t(x)$ is changing. Suppose $Q(t)$ keeps changing until some time $T_1$ and keeps constant, $Q(t)=Q$, after $T_1$,
\begin{equation}
  \begin{split}
    \lim_{T \rightarrow \infty} g_T(x)
    = e^{-T}  \int_{0}^{T_1} Q(t) e^{t} \,dt + e^{-T}  \int_{T_1}^{T} Q e^{t} \,dt .
  \end{split}
\end{equation}
As $T \rightarrow \infty$, the first part of right hand side converges to $0$.
\begin{equation}
  \begin{split}
    \lim_{T \rightarrow \infty} g_T(x)
    &\rightarrow e^{-T}  \int_{T_1}^{T} Q e^{t} \,dt \\
    &= e^{-T}  \int_{T_1}^{T} e^{t} \,dt \cdot  Q\\
    &= e^{-T} ( e^{T} - e^{T_1}) \cdot Q \\
    &\rightarrow Q \\
    &= C  \sum_{i = 1}^{n}   \mathbbm{1}(y_i g_T(x_i) < 1)   \cdot  y_i K(x, x_i) .
  \end{split}
\end{equation}

\subsection{Discrete Dynamics of SVM}
Let $\eta \in (0, 1)$ be the learning rate. The equation of subgradient descent update at some time $t$ is 
\begin{equation}
  \beta_{t+1} - \beta_t = - \eta \nabla_{\beta} L(\beta_t) . \\ 
\end{equation}
The dynamics of $g_t(x)$ is
\begin{equation}
  \begin{split}
    g_{t+1}(x)  - g_t(x)
    &= \inner{ \beta_{t+1} - \beta_{t}, \Phi(x) } \\
    &= \inner{ - \eta \beta_t + \eta C \sum_{i=1}^{n}  \mathbbm{1}(y_i g_t(x_i) < 1)  y_i  \Phi(x_i), \Phi(x) } \\
    &= -\eta g_t(x) + \eta C \sum_{i=1}^{n}  \mathbbm{1}(y_i g_t(x_i) < 1)  y_i  K(x, x_i) .
  \end{split}
\end{equation}
Denote second part as $Q(t) = \eta C \sum_{i=1}^{n}  \mathbbm{1}(y_i g_t(x_i) < 1)  y_i  K(x, x_i) $, which changes over time. The model $g_T(x)$ at some time $T$ is
\begin{equation}
\begin{split}
  g_T(x) 
  &= (1 - \eta) g_{T-1}(x)  + Q(T-1) \\
  &= (1 - \eta) \biggl( (1 - \eta) g_{T-2}(x)  + Q(T-2) \biggr) + Q(T-1) \\
  &= (1 - \eta)^{T}  g_0(x)  + \sum_{t=0}^{T-1}  (1 - \eta)^{T-1-t} Q(t) \\
  &= (1 - \eta)^{T}  \biggl( g_0(x)  + \sum_{t=0}^{T-1}  (1 - \eta)^{-t-1} Q(t) \biggr) .
\end{split}
\end{equation}

The convergence of subgradient descent usually requires additional assumption that the norm of the subgradient is bounded. We refer readers to \citep{shalev2011pegasos, boyd2003subgradient, bertsekas2015convex} for some proofs. Here let us assume the subgradient descent converges to the global optimizer and $Q(t)$ keeps changing until some time $T_1$ and keeps constant, $Q(t) = Q$, after $T_1$. As $T \rightarrow \infty$,
\begin{equation}
  \begin{split}
    g_T(x) 
    &\rightarrow \sum_{t=0}^{T-1}  (1 - \eta)^{T-1-t} Q(t) \\
    &= \sum_{t=0}^{T_1-1}  (1 - \eta)^{T-1-t} Q(t) + \sum_{t=T_1}^{T-1}  (1 - \eta)^{T-1-t} Q \\ 
    &\rightarrow \sum_{t=T_1}^{T-1}  (1 - \eta)^{T-1-t} Q \\ 
    &= \sum_{t=T_1}^{T-1}  (1 - \eta)^{T-1-t} Q \\ 
    &= \frac{-(1-\eta)^{T-T_1} + 1}{\eta} Q .
  \end{split} 
\end{equation}
As $\eta \in (0, 1)$, $-(1-\eta)^{T-T_1} \rightarrow 0$.
\begin{equation}
  \begin{split}
    g_T(x) 
    &\rightarrow \frac{1}{\eta} Q \\ 
    &= C \sum_{i=1}^{n}  \mathbbm{1}(y_i g_T(x_i) < 1)  y_i  K(x, x_i) . \\ 
  \end{split}
\end{equation}

\section{Dynamics and Convergence Rate of Neural Network Trained by Soft Margin Loss}\label{app:dy_nn}
\subsection{Continuous Dynamics of NN}
In the learning rate $\eta \rightarrow 0$ limit, the subgradient descent equation, which can also be written as 
\begin{equation}
    \frac{w_{t+1} - w_t}{\eta} = - \nabla_w L(w_t) ,
\end{equation}
becomes a differential equation
\begin{equation}
    \frac{d w_t}{dt} =  - \nabla_w L(w_t) .
\end{equation}
This is known as gradient flow \citep{ambrosio2008gradient}. Then for any differentiable function $f_t(x)$,
\begin{equation}
    \frac{d f_t(x)}{d t} = \sum_{j = 1}^{p} \frac{\partial f_t(x)}{\partial w_j} \frac{d w_j}{d t} ,
\end{equation}
where $p$ is the number of parameters. Replacing $\frac{d w_j}{d t} $ by its subgradient descent expression:
\begin{equation}
    \frac{d f_t(x)}{d t} = \sum_{j = 1}^{p} \frac{\partial f_t(x)}{\partial w_j} (- \frac{\partial L(w_t)}{\partial w_j} ).
\end{equation}
And we know 
\begin{equation}
    \frac{\partial L(w_t)}{\partial w_j} =  w_j \mathbbm{1}(w_j \in W^{(L+1)}) + \sum_{i = 1}^{n} \frac{\partial L_h}{\partial f_t(x_i)} \frac{\partial f_t(x_i)}{\partial w_j} .
\end{equation}
where $\mathbbm{1}(w_j \in W^{(L+1)})$ equals to 1 if the parameter $w_j$ is in the last layer $W^{(L+1)}$ else 0.
Combining above together, 
\begin{equation}
    \frac{d f_t(x)}{d t} = \sum_{j = 1}^{p} \frac{\partial f_t(x)}{\partial w_j} \biggl(- w_j \mathbbm{1}(w_j \in W^{(L+1)}) - \sum_{i = 1}^{n} \frac{\partial L_h}{\partial f_t(x_i)} \frac{\partial f_t(x_i)}{\partial w_j} \biggr) .
\end{equation}
Rearranging terms:
\begin{equation}
    \frac{d f_t(x)}{d t} = - \sum_{k = 1}^{p_{L+1}} \frac{\partial f_t(x)}{\partial W^{(L+1)}_k} W^{(L+1)}_k  
    - \sum_{i = 1}^{n} \frac{\partial L_h}{\partial f_t(x_i)} \sum_{j = 1}^{p} \frac{\partial f_t(x)}{\partial w_j} \frac{\partial f_t(x_i)}{\partial w_j} ,
\end{equation}
where $p_{L+1}$ is the number of parameters of the last layer ($L+1$ layer). The first part of the right hand side is
\begin{equation}
    \sum_{k = 1}^{p_{L+1}} \frac{\partial f_t(x)}{\partial W^{(L+1)}_k} W^{(L+1)}_k 
    = \inner{ \frac{\partial f_t(x)}{\partial W^{(L+1)}}, W^{(L+1)} } 
    = \inner{ \alpha_t^{(L)}(x), W^{(L+1)} }
    = f_t(x) .
\end{equation}
Applying $L_h'(y_i, f_t(x_i)) = \frac{\partial L_h}{\partial f_t(x_i)}$, the subgradient of hinge loss, and the definition of tangent kernel (\ref{tangent_k}), the second part is 
\begin{equation}
    - \sum_{i = 1}^{n} \frac{\partial L_h}{\partial f_t(x_i)} \sum_{j = 1}^{p} \frac{\partial f_t(x)}{\partial w_j} \frac{\partial f_t(x_i)}{\partial w_j} =  - \sum_{i = 1}^{n} L_h'(y_i, f_t(x_i)) \hat{\Theta}(w_t; x, x_i) .
\end{equation}
Thus the equation becomes
\begin{equation}\label{eq:diff}
    \frac{d f_t(x)}{d t} = - f_t(x) - \sum_{i = 1}^{n} L_h'(y_i, f_t(x_i)) \hat{\Theta}(w_t; x, x_i) .
\end{equation}
Take $L_h'(y_i, f_t(x_i)) = - C y_i  \mathbbm{1}(y_i f_t(x_i) < 1) $ in
\begin{equation}
  \frac{d f_t(x)}{d t} =  - f_t(x) + C \sum_{i = 1}^{n} \mathbbm{1}(y_i f_t(x_i) < 1)  y_i  \hat{\Theta}(w_t; x, x_i) .
\end{equation}

\subsection{Additional Notations}
Denote $X \in \mathbb{R}^{d \times n}$ as the training data. Denote $f_t = f_t(X) \in \mathbb{R}^n$ and $g_t = g_t(X) \in \mathbb{R}^n$ as the outputs of NN and SVM on the training data. Denote $\hat{\Theta}(w_t) = \hat{\Theta}(w_t; X, X) \in \mathbb{R}^{n \times n}$ as the tangent kernel evaluated on the training data at time $t$, and $l'(f_t) \in \mathbb{R}^n$ as the derivative of the loss function w.r.t. $f_t$. Denote $\nabla_w f_t \in \mathbb{R}^{n \times p}$ as the Jacobian and we have $\hat{\Theta}(w_t) = \nabla_w f_t \nabla_w f_t^T$. Denote $\lambda_0 = \lambda_{min}\parentheses{\hat{\Theta}(w_t)}$ as the smallest eigenvalue of $\hat{\Theta}(w_t)$. Then we can write the dynamics of NN as
\begin{gather*}
    \frac{d}{dt} f_t = - f_t - \hat{\Theta}(w_t) l'(f_t) .
\end{gather*}

Let $v \in \mathbb{R}^p$ with $v_j = \mathbbm{1}(w_j \in W^{(L+1)})$. We can write the gradient as
\begin{gather*}
    \nabla_w L(w_t) = w_t \odot v + \nabla_w f_t^T l'(f_t) .
\end{gather*}

\subsection{Convergence of NN}

The loss of NN is
\begin{equation*}
    L(w_t) = \frac{1}{2}\norm{W_t^{(L+1)}}^2 + \sum_i^n l(f_t(x_i), y_i) ,
\end{equation*}
where $l(f, y) = C\max(0, 1-yf)$. The dynamic of the loss is
\begin{equation*}
    \frac{d L(w_t) }{d t} 
    = \frac{\partial L(w_t) }{\partial w_t} \frac{d w_t }{d t} 
    = \inner{\nabla_{w} L(w_t), - \nabla_{w} L(w_t) }
    = - \norm{\nabla_{w} L(w_t) }^2 .
\end{equation*}
Since $L(w_t) \geq 0$ is bounded from below, by monotone convergence theorem, $L(w_t)$ will always converge to a stationary point. Applying Lemma~\ref{lemma:pl_nn}, we have 
\begin{align*}
    \frac{d \parentheses{L(w_t) - L(w^*)}}{d t} 
    = - \norm{\nabla_{w} L(w_t) }^2 
    \leq -2\parentheses{L(w_t) - L(w^*)} .
\end{align*}
Thus we have a linear convergence, same as SVM.
\begin{equation*}
    L(w_t) - L(w^*) \leq e^{-2t}\parentheses{L(w_0) - L(w^*)} .
\end{equation*}

\begin{lemma}[PL inequality of NN for soft margin loss] \label{lemma:pl_nn}
  Assume $\lambda_0 \geq \frac{2}{C}$, then  $L(w_t)$ satisfies the PL condition 
  \begin{equation*}
      \norm{\nabla_{w} L(w_t) }^2  \geq 2\parentheses{L(w_t) - L(w^*)} .
  \end{equation*}
\end{lemma}
\begin{proof}
\begin{align*}
    \norm{\nabla_{w} L(w_t) }^2 
    &= \inner{w_t \odot v + \nabla_w f_t^T l'(f_t), w_t \odot v + \nabla_w f_t^T l'(f_t)} \\
    &= \inner{w_t \odot v, w_t \odot v} + \inner{\nabla_w f_t^T l'(f_t), \nabla_w f_t^T l'(f_t)} + 2\inner{w_t \odot v, \nabla_w f_t^T l'(f_t)} \\
    &= \norm{W_t^{(L+1)}}^2 + l'(f_t)^T \hat{\Theta}(w_t) l'(f_t) + 2\inner{W_t^{(L+1)}, \nabla_{W^{(L+1)}} f_t^T l'(f_t)} \\
    &= \norm{W_t^{(L+1)}}^2 + l'(f_t)^T \hat{\Theta}(w_t) l'(f_t) + 2 f_t^T l'(f_t) .
\end{align*}

We want the loss satisfies the PL condition $\norm{\nabla_{w} L(w_t) }^2  \geq 2\parentheses{L(w_t) - L(w^*)}$.
\begin{align*}
    &\norm{\nabla_{w} L(w_t) }^2  - 2\parentheses{L(w_t) - L(w^*)} \\
    &= \norm{\nabla_{w} L(w_t) }^2  - 2L(w_t) + 2L(w^*) \\
    &= l'(f_t)^T \hat{\Theta}(w_t) l'(f_t) + 2 f_t^T l'(f_t) - 2\sum_i^n l(f_t(x_i), y_i) + 2L(w^*) \\
    &\geq \lambda_0 \norm{l'(f_t)}^2 + 2 f_t^T l'(f_t) - 2\sum_i^n l(f_t(x_i), y_i) + 2L(w^*) ,
\end{align*}
where the last inequality is the inequality of quadratic form. For hinge loss $l(f, y) = C\max(0, 1-yf) = C(1-yf)\mathbbm{1}(1-yf>0)$ and $l'(f, y) = -Cy\mathbbm{1}(1-yf>0)$,
\begin{align*}
    &\norm{\nabla_{w} L(w_t) }^2  - 2\parentheses{L(w_t) - L(w^*)} \\
    &\geq \lambda_0 \norm{l'(f_t)}^2 + 2 f_t^T l'(f_t) - 2\sum_i^n l(f_t(x_i), y_i) + 2L(w^*) \\
    &= \lambda_0 \sum_i^n {l'(f_t(x_i), y_i)}^2 + 2 \sum_i^n f_t(x_i) l'(f_t(x_i), y_i) - 2\sum_i^n l(f_t(x_i), y_i) + 2L(w^*) \\
    &= \lambda_0 \sum_i^n C^2\mathbbm{1}(1-y_if_t(x_i)>0) - 2 \sum_i^n Cy_if_t(x_i) \mathbbm{1}(1-y_if_t(x_i)>0) \\ 
    &\qquad - 2\sum_i^n C(1-y_if_t(x_i))\mathbbm{1}(1-y_if_t(x_i)>0) + 2L(w^*) \\
    &= C\sum_i^n \mathbbm{1}(1-y_if_t(x_i)>0)\parentheses{C\lambda_0 - 2}  + 2L(w^*) .
\end{align*}
Since $L(w^*) > 0$, as long as $\lambda_0 \geq 2 / C$, the loss $L(w_t)$ satisfies the PL condition $\norm{\nabla_{w} L(w_t) }^2  \geq 2\parentheses{L(w_t) - L(w^*)}$. $\lambda_0 \geq 2 / C$ can be guaranteed in a parameter ball when $\frac{2}{C} < \lambda_{min}\parentheses{\hat{\Theta}(w_0)}$ by using a sufficiently wide NN \citep{liu2020loss}.
\end{proof}

\subsection{Discrete Dynamics of NN}\label{app:dis_nn} 
The subgradient descent update is 
\begin{equation}
    w_{t+1} - w_t = - \eta \nabla_w L(w_t) .
\end{equation}
We consider the situation of constant NTK, $\hat{\Theta}(w_t; x, x_i) \rightarrow \hat{\Theta}(w_0; x, x_i)$, or equivalently linear model. As proved by Proposition 2.2 in \citep{liu2020linearity}, the tangent kernel of a differentiable
function $f(w, x)$ is constant if and only if $f(w, x)$ is linear in $w$. Take the Taylor expansion of $f(w_{t+1}, x)$ at $w_{t}$,
\begin{equation}
\begin{split}
    &\quad f(w_{t+1}, x) - f(w_{t}, x) \\ 
    &= f(w_{t}, x) + \inner{\nabla_w f(w_{t}, x), w_{t+1}-w_{t}} - f(w_{t}, x) \\
    &=\inner{\nabla_w f(w_{t}, x), - \eta \nabla_w L(w_t)} \\
    &= \inner{\nabla_w f(w_{t}, x), - \eta \Big( w v +  \sum_{i = 1}^{n} L_h'(y_i, f_t(x_i)) \nabla_w f_t(x_i)  \Big)}\\
    &= - \eta f_t(x) + \eta  \sum_{i = 1}^{n} L_h'(y_i, f_t(x_i)) \hat{\Theta}(w_t; x, x_i)\\
    &= - \eta f_t(x) + \eta  C \sum_{i = 1}^{n} \mathbbm{1}(y_i f_t(x_i) < 1)  y_i  \hat{\Theta}(w_t; x, x_i)\\
    &\rightarrow - \eta f_t(x) + \eta  C \sum_{i = 1}^{n} \mathbbm{1}(y_i f_t(x_i) < 1)  y_i  \hat{\Theta}(w_0; x, x_i) .
\end{split}
\end{equation}

\section{Proof of Theorem \ref{theorem:equivalence}}\label{app:equivalence}
\begin{proof}
We prove the constancy of tangent kernel by adopting the results of \citep{liu2016large}.
\begin{lemma}[Theorem 3.3 in \citep{liu2020linearity}; Hessian norm is controlled by the minimum hidden layer width] \label{lemma:hessian_bound}
  Consider a general neural network $f(w, x)$ of the form Eq. (\ref{eq:NN}), which can be a fully connected network, CNN, ResNet or a mixture of these types. Let m be the minimum of the hidden layer widths, i.e., $m = \min_{l \in [L]} m_l$. Given any fixed $R > 0$, and any $w \in B(w_0; R) := \{ w: \norm{ w - w_0 } \leq R \}$, with high probability over the initialization, the Hessian spectral norm satisfies the following:
  \begin{equation}
    \norm{ H(w) } = O(\frac{R^{3L}\ln{m}}{\sqrt{m}}). 
  \end{equation}
\end{lemma}

\begin{lemma}[Proposition 2.3 in \citep{liu2020linearity}; Small Hessian norm $\Rightarrow$ Small change of tangent kernel] \label{lemma:ntk_bound}
  Given a point $w_0 \in \mathbb{R}^p$ and a ball $B(w_0; R) := \{ w: \norm{ w - w_0 } \leq R \}$ with fixed radius $R > 0$, if the Hessian matrix satisfies $ \norm{ H(w) } < \epsilon $, where $\epsilon > 0$, for all $w \in B(w_0, R)$, then the tangent kernel $\hat{\Theta}(w; x, x')$ of the model, as a function of $w$, satisfies
  \begin{equation}
    \abs{ \hat{\Theta}(w; x, x') -  \hat{\Theta}(w_0; x, x') } = O(\epsilon R), \quad \forall w \in B(w_0; R),\ \forall x, x' \in \mathbb{R}^d.
  \end{equation}
\end{lemma}\label{hessian_control}

Applying above two lemmas, we can see that in the limit of $m \rightarrow \infty$, the spectral norm of Hessian converge to $0$ and the tangent kernel keeps constant in the ball $B(w_0; R)$. 
\begin{corollary}[Consistancy of tangent kernel]
  Consider a general neural network $f(w, x)$ of the form Eq. (\ref{eq:NN}). Given a point $w_0 \in \mathbb{R}^p$ and a ball $B(w_0; R) := \{ w: \norm{ w - w_0 } \leq R \}$ with fixed radius $R > 0$, in the infinite width limit, $m \rightarrow \infty$,
  \begin{equation}
    \lim_{m \rightarrow \infty} \hat{\Theta}(w; x, x') \rightarrow  \hat{\Theta}(w_0; x, x_i) , \quad \forall w \in B(w_0; R),\ \forall x, x' \in \mathbb{R}^d.
  \end{equation}
\end{corollary}
Thus we prove the constancy of tangent kernel in infinite width limit. Then it is easy to check the dynamics of infinitely wide NN is the same with the dynamics of SVM with constant NTK.

\end{proof}

\section{Bound the difference between SVM and NN}\label{app:bound_output}
Assume the loss $l$ is $\rho$-lipschitz and $\beta_l$-smooth for the first argument (i.e. the model output). Assume $f_0(x) = g_0(x)$ for any $x$.

\subsection{Bound the difference on the Training Data}
The dynamics of the NN and SVM are
\begin{gather*}
    \frac{d}{dt} f_t = - \lambda f_t - \hat{\Theta}(w_t) l'(f_t) \\
    \frac{d}{dt} g_t = - \lambda g_t - \hat{\Theta}(w_0) l'(g_t)
\end{gather*}
The dynamics of the difference between them is
\begin{equation*}
    \frac{d}{dt} \parentheses{f_t - g_t} = - \lambda \parentheses{f_t - g_t} - \parentheses{\hat{\Theta}(w_t) l'(f_t) - \hat{\Theta}(w_0) l'(g_t)} 
\end{equation*}
The solution of the above differential equation at time $T$ is
\begin{align*}
    f_T - g_T 
    &= e^{-\lambda T}\parentheses{f_0 - g_0} - e^{-\lambda T} \int_0^T \parentheses{\hat{\Theta}(w_t) l'(f_t) - \hat{\Theta}(w_0) l'(g_t)} e^{\lambda t} dt \\
    &= e^{-\lambda T} \int_0^T \parentheses{\hat{\Theta}(w_0) l'(g_t) - \hat{\Theta}(w_t) l'(f_t)} e^{\lambda t} dt
\end{align*}
using $f_0 = g_0$. Thus
\begin{align*}
    \norm{f_T - g_T}
    &\leq e^{-\lambda T} \int_0^T \norm{\hat{\Theta}(w_0) l'(g_t) - \hat{\Theta}(w_t) l'(f_t)} e^{\lambda t} dt
\end{align*}

Since $l$ is $\beta_l$ smooth,
\begin{align*}
    \norm{\hat{\Theta}(w_0) l'(g_t) - \hat{\Theta}(w_t) l'(f_t)}
    &= \norm{\hat{\Theta}(w_0) l'(g_t) - \hat{\Theta}(w_0) l'(f_t) + \hat{\Theta}(w_0) l'(f_t) - \hat{\Theta}(w_t) l'(f_t)} \\
    &= \norm{\hat{\Theta}(w_0) \parentheses{l'(g_t) - l'(f_t)} + \parentheses{\hat{\Theta}(w_0) - \hat{\Theta}(w_t)} l'(f_t)} \\
    &\leq \norm{\hat{\Theta}(w_0) \parentheses{l'(g_t) - l'(f_t)}} + \norm{\parentheses{\hat{\Theta}(w_0) - \hat{\Theta}(w_t)} l'(f_t)} \\
    &\leq \beta_l \norm{\hat{\Theta}(w_0)} \norm{g_t - f_t} + \rho \sqrt{n} \norm{\hat{\Theta}(w_0) - \hat{\Theta}(w_t)} 
\end{align*}
where $\norm{l'(f_t)} \leq \rho \sqrt{n}$. Thus we have 
\begin{align*}
    \norm{f_T - g_T }
    &\leq e^{-\lambda T} \beta_l \norm{\hat{\Theta}(w_0)} \int_0^T  \norm{g_t - f_t} e^{\lambda t} dt + e^{-\lambda T} \rho \sqrt{n} \int_0^T  \norm{\hat{\Theta}(w_0) - \hat{\Theta}(w_t)} e^{\lambda t} dt
\end{align*}
Applying the Grönwall's inequality,
\begin{align*}
    \norm{f_T - g_T } 
    &\leq e^{-\lambda T} \rho \sqrt{n} \int_0^T  \norm{\hat{\Theta}(w_0) - \hat{\Theta}(w_t)} e^{\lambda t} dt \cdot e^{e^{-\lambda T} \beta_l \norm{\hat{\Theta}(w_0)} \int_0^T e^{\lambda t} dt} \\
    &= e^{-\lambda T} \rho \sqrt{n} \int_0^T  \norm{\hat{\Theta}(w_0) - \hat{\Theta}(w_t)} e^{\lambda t} dt \cdot e^{\frac{1}{\lambda}(1- e^{- \lambda T}) \beta_l \norm{\hat{\Theta}(w_0)}} \\
    &= e^{-\lambda T} e^{\frac{1}{\lambda}(1- e^{- \lambda T}) \beta_l \norm{\hat{\Theta}(w_0)}} \rho \sqrt{n} \int_0^T  \norm{\hat{\Theta}(w_0) - \hat{\Theta}(w_t)} e^{\lambda t} dt 
\end{align*}
By Lemma~\ref{lemma:hessian_bound} and Lemma~\ref{lemma:ntk_bound}, in a parameter ball $B(w_0; R) = \{w: \norm{w-w_0} \leq R\}$, with high probability, $\abs{\hat{\Theta}(w; x, x') -  \hat{\Theta}(w_0; x, x') } = O(R^{3L+1} \ln{m}/ \sqrt{m})$ w.r.t. $m$. Then we have
\begin{align*}
    \norm{\hat{\Theta}(w_0) - \hat{\Theta}(w_t)} \leq \norm{\hat{\Theta}(w_0) - \hat{\Theta}(w_t)}_F = O(\frac{ R^{3L+1} n \ln{m}}{\sqrt{m}})
\end{align*}
Thus we have
\begin{align*}
    \norm{f_T - g_T } 
    &\leq \frac{1}{\lambda}(1-e^{-\lambda T}) e^{(1- e^{-T}) \beta_l \norm{\hat{\Theta}(w_0)}} \rho \sqrt{n} \cdot  O(\frac{ R^{3L+1} n \ln{m}}{\sqrt{m}}) \\
    &= O(\frac{ e^{\beta_l \norm{\hat{\Theta}(w_0)}} R^{3L+1} \rho n^{\frac{3}{2}} \ln{m}}{\lambda \sqrt{m}})
\end{align*}


\subsection{Bound on the Test Data}
For a test data $x$, the prove is similar to the training case. Denote $\hat{\Theta}(w_t; X, x) \in \mathbb{R}^{n}$ as the tangent kernel evaluate between the training data and a test data $x$. Recall 
\begin{gather*}
     \frac{d f_t(x)}{d t} =  - \lambda f_t(x) - \hat{\Theta}(w_t; X, x)^T l'(f_t)  \\
     \frac{d g_t(x)}{d t} =  - \lambda g_t(x) - \hat{\Theta}(w_0; X, x)^T l'(g_t) \\
     \frac{d}{dt} \parentheses{f_t(x) - g_t(x)} = - \lambda \parentheses{f_t(x) - g_t(x) } - \parentheses{ \hat{\Theta}(w_t; X, x)^T l'(f_t) - \hat{\Theta}(w_0; X, x)^T l'(g_t)}
\end{gather*}

The solution of the above differential equation is 
\begin{align*}
    f_T(x) - g_T(x) 
    &= e^{-\lambda T}\parentheses{f_0 - g_0} - e^{-\lambda T} \int_0^T \parentheses{\hat{\Theta}(w_t; X, x)^T l'(f_t) - \hat{\Theta}(w_0; X, x)^T l'(g_t)} e^{\lambda t} dt \\
    &= e^{-\lambda T} \int_0^T \parentheses{\hat{\Theta}(w_0; X, x)^T l'(g_t) - \hat{\Theta}(w_t; X, x)^T l'(f_t)} e^{\lambda t} dt
\end{align*}
using $f_0 = g_0$. Thus
\begin{align*}
    \norm{f_T(x) - g_T(x)}
    &\leq e^{-\lambda T} \int_0^T \norm{\hat{\Theta}(w_0; X, x)^T l'(g_t) - \hat{\Theta}(w_t; X, x)^T l'(f_t)} e^{\lambda t} dt
\end{align*}

Since $l$ is $\beta_l$ smooth,
\begin{align*}
    &\norm{\hat{\Theta}(w_0; X, x)^T l'(g_t) - \hat{\Theta}(w_t; X, x)^T l'(f_t)} \\
    &= \norm{\hat{\Theta}(w_0; X, x)^T l'(g_t) - \hat{\Theta}(w_0; X, x)^T l'(f_t) + \hat{\Theta}(w_0; X, x)^T l'(f_t) - \hat{\Theta}(w_t; X, x)^T l'(f_t)} \\
    &= \norm{\hat{\Theta}(w_0; X, x)^T \parentheses{l'(g_t) - l'(f_t)} + \parentheses{\hat{\Theta}(w_0; X, x)^T - \hat{\Theta}(w_t; X, x)^T} l'(f_t)} \\
    &\leq \norm{\hat{\Theta}(w_0; X, x)^T \parentheses{l'(g_t) - l'(f_t)}} + \norm{ \parentheses{\hat{\Theta}(w_0; X, x)^T - \hat{\Theta}(w_t; X, x)^T} l'(f_t)} \\
    &\leq \beta_l \norm{\hat{\Theta}(w_0; X, x)} \norm{g_t - f_t} + \rho \sqrt{n} \norm{ \parentheses{\hat{\Theta}(w_0; X, x)^T - \hat{\Theta}(w_t; X, x)^T} }
\end{align*}
where $\norm{l'(f_t)} \leq \rho \sqrt{n}$. Thus we have
\begin{align*}
    &\norm{f_T(x) - g_T(x) } \\
    &\leq e^{-\lambda T} \beta_l \norm{\hat{\Theta}(w_0; X, x)} \int_0^T  \norm{g_t - f_t} e^{\lambda t} dt + e^{-\lambda T} \rho \sqrt{n} \int_0^T  \norm{\hat{\Theta}(w_0; X, x)^T - \hat{\Theta}(w_t; X, x)^T}  e^{\lambda t} dt
\end{align*}
 Applying the Grönwall's inequality,
\begin{align*}
    &\norm{f_T(x) - g_T(x) } \\
    &\leq e^{-\lambda T} \rho \sqrt{n} \int_0^T  \norm{\hat{\Theta}(w_0; X, x)^T - \hat{\Theta}(w_t; X, x)^T} e^{\lambda t} dt \cdot e^{e^{-\lambda T} \beta_l \norm{\hat{\Theta}(w_0; X, x)} \int_0^T e^{\lambda t} dt} \\
    &= e^{-\lambda T} \rho \sqrt{n} \int_0^T  \norm{\hat{\Theta}(w_0; X, x)^T - \hat{\Theta}(w_t; X, x)^T} e^{\lambda t} dt \cdot e^{\frac{1}{\lambda}(1- e^{- \lambda T}) \beta_l \norm{\hat{\Theta}(w_0; X, x)}} \\
    &= e^{-\lambda T} e^{\frac{1}{\lambda}(1- e^{- \lambda T}) \beta_l \norm{\hat{\Theta}(w_0; X, x)}} \rho \sqrt{n} \int_0^T  \norm{\hat{\Theta}(w_0; X, x)^T - \hat{\Theta}(w_t; X, x)^T} e^{\lambda t} dt 
\end{align*}

By Lemma~\ref{lemma:hessian_bound} and Lemma~\ref{lemma:ntk_bound}, in a parameter ball $B(w_0; R) = \{w: \norm{w-w_0} \leq R\}$, with high probability, $\abs{\hat{\Theta}(w; x, x') -  \hat{\Theta}(w_0; x, x') } = O(R^{3L+1} \ln{m} / \sqrt{m})$. Then we have
\begin{align*}
    \norm{\hat{\Theta}(w_0; X, x)^T - \hat{\Theta}(w_t; X, x)^T} = O(\frac{ R^{3L+1} \sqrt{n} \ln{m}}{\sqrt{m}})
\end{align*}
Thus we have
\begin{align*}
    \norm{f_T(x) - g_T(x) } 
    &\leq \frac{1}{\lambda}(1-e^{-\lambda T}) e^{(1- e^{-T}) \beta_l \norm{\hat{\Theta}(w_0; X, x)}} \rho \sqrt{n} \cdot  O(\frac{ R^{3L+1} \sqrt{n} \ln{m}}{\sqrt{m}}) \\
    &= O(\frac{ e^{\beta_l \norm{\hat{\Theta}(w_0; X, x)}} R^{3L+1} \rho n \ln{m}}{\lambda \sqrt{m}})
\end{align*}

\section{Finite-width Neural Networks are Kernel Machines}\label{app:kernel_machine}
Inspired by \citep{domingos2020every}, we can also show that every neural network trained by (sub)gradient descent with loss function in the form (\ref{eq:loss_svm2}) is approximately a kernel machine without the assumption of infinite width limit.
\begin{theorem} 
  Suppose a neural network $f(w, x)$, with $f$ a differentiable function of $w$, is learned from a training set $\{(x_i, y_i)\}_{i=1}^{n}$ by (sub)gradient descent with loss function $L(w) = \frac{\lambda}{2}\norm{ W^{(L+1)} }^2 + \sum_{i = 1}^{n} l(f(w, x_i), y_i)$ and gradient flow. Assume $\text{sign}(l'(f_t(x_i), y_i)) = \text{sign}(l'(f_0(x_i), y_i)), \forall t \in [0, T]$, keeps unchanged during training. Then at some time $T$,
  \begin{equation}
      f_T(x) = \sum_{i = 1}^{n} a_i K(x, x_i) + b,
  \end{equation}
  where 
\begin{gather*}
    a_i = - \text{sign}(l'(f_0(x_i), y_i)) , \qquad 
    b = e^{-\lambda T} f_0(x),   \\
    K(x, x_i) = e^{-\lambda T} \int_{0}^{T} \abs{l'(f_t(x_i), y_i)} \hat{\Theta}(w_t; x, x_i) e^{\lambda t} \,dt   
\end{gather*}
\end{theorem}

\begin{proof}
As we have derived, the neural network follows the dynamics of Eq. (\ref{eq:dy_general}):
\begin{equation}
  \frac{d f_t(x)}{d t} =  - \lambda f_t(x) - \sum_{i = 1}^{n} l'(f_t(x_i), y_i) \hat{\Theta}(w_t; x, x_i).
\end{equation} 

Note this is a  first-order inhomogeneous linear differential equation with the functions depended on $t$. Denote $Q(t) = - \sum_{i = 1}^{n} l'(f_t(x_i), y_i) \hat{\Theta}(w_t; x, x_i)$,
\begin{equation}
    \frac{d f_t(x)}{d t} + \lambda f_t(x) = Q(t) .
\end{equation}
Let $f_0(x)$ be the initial model, prior to gradient descent. The solution is given by
\begin{equation}
    f_T(x) = e^{-\lambda T} \biggl( f_0(x)  + \int_{0}^{T} Q(t) e^{\lambda t} \,dt \biggr) .
\end{equation}
Then
\begin{equation}
\begin{split}
    f_T(x)
    &= e^{-\lambda T} \biggl( f_0(x)  - \sum_{i = 1}^{n} \int_{0}^{T} l'(f_t(x_i), y_i) \hat{\Theta}(w_t; x, x_i) e^{\lambda t} \,dt \biggr)  \\
    &= e^{-\lambda T} f_0(x)  -  \sum_{i = 1}^{n}  e^{-\lambda T} \int_{0}^{T}  l'(f_t(x_i), y_i) \hat{\Theta}(w_t; x, x_i) e^{\lambda t} \,dt   \\
    &= e^{-\lambda T} f_0(x)  -  \sum_{i = 1}^{n} e^{-\lambda T} \int_{0}^{T}  \text{sign}(l'(f_t(x_i), y_i)) \cdot  \abs{l'(f_t(x_i), y_i)} \hat{\Theta}(w_t; x, x_i) e^{\lambda t} \,dt   \\
    &= e^{-\lambda T} f_0(x)  -  \sum_{i = 1}^{n} \text{sign}(l'(f_0(x_i), y_i)) \cdot e^{-\lambda T} \int_{0}^{T} \abs{l'(f_t(x_i), y_i)} \hat{\Theta}(w_t; x, x_i) e^{\lambda t} \,dt   .
\end{split}
\end{equation}
where the last equality uses the assumption $\text{sign}(l'(f_t(x_i), y_i)) = \text{sign}(l'(f_0(x_i), y_i)), \forall t \in [0, T]$. Thus 
\begin{equation}
    f_T(x)
    = \sum_{i = 1}^{n} a_i K(x, x_i) + b,
\end{equation}
with 
\begin{gather*}
    a_i = - \text{sign}(l'(f_0(x_i), y_i)) , \qquad 
    b = e^{-\lambda T} f_0(x),   \\
    K(x, x_i) = e^{-\lambda T} \int_{0}^{T} \abs{l'(f_t(x_i), y_i)} \hat{\Theta}(w_t; x, x_i) e^{\lambda t} \,dt 
\end{gather*}
\end{proof}
$K(x, x_i) = e^{-\lambda T} \int_{0}^{T} \abs{l'(f_t(x_i), y_i)} \hat{\Theta}(w_t; x, x_i) e^{\lambda t} \,dt  $ is a valid kernel since it is a nonnegative sum of positive definite kernels. Our $a_i$, $b$ and $K(x, x_i)$ will stay bounded as long as $f_0(x)$, $l'(f_t(x_i), y_i)$ and $\hat{\Theta}(w_t; x, x_i)$ are bounded.

\section{Robustness of Over-parameterized Neural Network}
\subsection{Robustness Verification of NTK}\label{app:robust}
For an infinitely wide two-layer fully connected ReLU NN, $f(x) = \frac{1}{\sqrt{m}}\sum_{j=1}^m v_j \sigma(\frac{1}{\sqrt{d}}w_j^T x)$, where $\sigma(z) = \max(0, z)$ is the ReLU activation. The NTK is
\begin{gather}
\begin{split}
    \Theta(x, x') 
    = \frac{\inner{x, x'}}{d}(\frac{\pi - \arccos(u)}{\pi}) + \frac{\norm{x}\norm{x'}}{2\pi d} \sqrt{1-u^2}
    = \frac{\norm{x}\norm{x'}}{2\pi d} h(u),
\end{split}\\
    h(u) = 2u(\pi - \arccos(u)) +  \sqrt{1-u^2}.
\end{gather}
where $u = \frac{ \inner{x, x'}}{\norm{x}\norm{x'}} \in \left[-1, 1\right]$. Consider the $\ell_{\infty}$ perturbation, for $x \in B_{\infty}(x_0, \delta) = \{x \in \mathbb{R}^d: \norm{x - x_0}_{\infty} \leq \delta \}$, we can bound $\norm{x}$ in the interval $[\norm{x}^L, \norm{x}^U]$ as follows.
\begin{gather*}
    \norm{x} 
    = \norm{x_0 + \Delta} 
    \leq \norm{x_0} + \norm{\Delta} 
    \leq \norm{x_0} + \sqrt{d} \delta = \norm{x}^U , \\
    \norm{x} 
    = \norm{x_0 + \Delta} 
    \geq \abs{\norm{x_0} - \norm{\Delta} }
    \geq \max(\norm{x_0} - \sqrt{d} \delta, 0) = \norm{x}^L .
\end{gather*} 
Then we can also bound $u$ in $[u^L, u^U]$.
\begin{gather*}
    \inner{x, x'} = \inner{x_0 + \Delta, x'} \in \left[\inner{x_0, x'} - \sqrt{d} \delta \norm{x'},  \inner{x_0, x'} + \sqrt{d} \delta \norm{x'} \right] ,\\
    u^L = \frac{ \inner{x_0, x'} - \sqrt{d} \delta \norm{x'}}{\norm{x}^U\norm{x'}} 
    \quad \text{if} \ \inner{x_0, x'} - \sqrt{d} \delta \norm{x'} \geq 0
    \quad \text{else} \quad \frac{ \inner{x_0, x'} - \sqrt{d} \delta \norm{x'}}{\norm{x}^L\norm{x'}} ,\\
    u^U = \frac{ \inner{x_0, x'} + \sqrt{d} \delta \norm{x'}}{\norm{x}^L\norm{x'}} 
    \quad \text{if} \ \inner{x_0, x'} + \sqrt{d} \delta \norm{x'} \geq 0
    \quad \text{else} \quad \frac{ \inner{x_0, x'} + \sqrt{d} \delta \norm{x'}}{\norm{x}^U\norm{x'}} ,\\
    u^U = \min(u^U, 1) .
\end{gather*}
where $\Delta \in B_{\infty}(0, \delta)$. $h(u)$ is a bow shaped function so it is easy to get its interval $[h^L(u), h^U(u)]$. Then we can get the interval of $\Theta(x, x') $, denote as $[\Theta^L(x, x'), \Theta^U(x, x')]$. 
\begin{gather*}
    \Theta^L(x, x') = \frac{\norm{x}^L\norm{x'}}{2\pi d} h^L(u)
    \quad \text{if} \ h^L(u) \geq 0 
    \quad \text{else} \quad  \frac{\norm{x}^U\norm{x'}}{2\pi d} h^L(u) ,\\
    \Theta^U(x, x') = \frac{\norm{x}^U\norm{x'}}{2\pi d} h^U(u)
    \quad \text{if} \ h^U(u) \geq 0 
    \quad \text{else} \quad  \frac{\norm{x}^L\norm{x'}}{2\pi d} h^U(u) .
\end{gather*}
Suppose the $g(x) = \sum_{i = 1}^{n}  \alpha_i \Theta(x, x_i) $, $\alpha_i$ are known after solving the kernel machine problem. Then we can lower bound and upper bound $g(x)$ as follows.
\begin{gather}
    g(x) \geq \sum_{i = 1, \alpha_i >0}^{n}  \alpha_i \Theta^L(x, x_i) + \sum_{i = 1, \alpha_i<0}^{n}  \alpha_i \Theta^U(x, x_i) ,\\
    g(x) \leq \sum_{i = 1, \alpha_i <0}^{n}  \alpha_i \Theta^L(x, x_i) + \sum_{i = 1, \alpha_i>0}^{n}  \alpha_i \Theta^U(x, x_i) .
\end{gather}

\subsection{IBP for Two-layer Neural Network}\label{app:ibp}
See the computation of IBP in \citep{gowal2018effectiveness}. For affine layers of NTK parameterization, the IBP bounds are computed as follows.
\begin{equation}
    \begin{split}
    \mu_{k-1} &= \frac{\overline{z}_{k-1} + \underline{z}_{k-1} }{2} \\
    r_{k-1} &= \frac{\overline{z}_{k-1} - \underline{z}_{k-1} }{2} \\
    \mu_{k} &= \frac{1}{\sqrt{m}} W \mu_{k-1} + b \\
    r_{k}   &= \frac{1}{\sqrt{m}} \abs{W} r_{k-1} \\
    \underline{z}_{k} &= \mu_{k} - r_{k} \\
    \overline{z}_{k}  &= \mu_{k} + r_{k} 
    \end{split}
\end{equation}

where $m$ is the input dimension of that layer. At initialization, $W$, $\mu_{k-1}$ and $b$ are independent. Since $\expect{\bracket{W} } = 0$ and $\expect{\bracket{b} } = 0$,
\begin{equation}
    \expect{\bracket{\mu_{k} }} 
    = \frac{1}{\sqrt{m}} \expect{\bracket{W} } \expect{\bracket{\mu_{k-1}}} + \expect{ \bracket{b}} 
    = 0
\end{equation}

Since $\abs{W}$ follows a folded normal distribution (absolute value of normal distribution) and $r_{k-1} \geq 0$, $\abs{W}  \geq 0$, $\expect{\bracket{\abs{W} } } \expect{\bracket{r_{k-1}}} = O(m)$,
\begin{equation}
    \expect{\bracket{r_{k} }} 
    = \frac{1}{\sqrt{m}} \expect{\bracket{\abs{W} } } \expect{\bracket{r_{k-1}}} 
    = O(\sqrt{m})
\end{equation}
Thus
\begin{gather}
    - \expect{\bracket{\underline{z}_{k} }} 
    = - \expect{\bracket{\mu_{k} }}  +  \expect{\bracket{r_{k} }} 
    = O(\sqrt{m}) \\
    \expect{\bracket{\overline{z}_{k} }}
    = \expect{\bracket{\mu_{k} }}  +  \expect{\bracket{r_{k} }}
    = O(\sqrt{m})
\end{gather}

And this will cause the robustness lower bound to decrease at a rate of $O(1/\sqrt{m})$. The same results hold for LeCun initialization, which is used in PyTorch for fully connected layers by default.

\end{document}